\documentclass{article}
\usepackage{fullpage}

\usepackage{cite}

\usepackage[utf8]{inputenc} 
\usepackage[T1]{fontenc}    
\usepackage{hyperref}       
\usepackage{url}            
\usepackage{booktabs}       
\usepackage{amsfonts}       
\usepackage{nicefrac}       
\usepackage{microtype}      
\usepackage{xcolor}         
\usepackage{amssymb,amsmath,amsthm}
\usepackage{pgfplots}
\pgfplotsset{compat=1.15}
\usetikzlibrary{pgfplots.fillbetween} 

\usepackage{paralist}
\allowdisplaybreaks

\newtheorem{theorem}{Theorem}
\newtheorem{proposition}[theorem]{Proposition}
\newtheorem{lemma}[theorem]{Lemma}

\newtheorem{corollary}[theorem]{Corollary}

\newtheorem*{definition}{Definition}

\newtheorem{example}{Example}
\newtheorem*{remark}{Remark}

\newcommand{\prob}[1]{\mathrm{Pr}\left(#1\right)}

\newcommand{\expecto}[2]{\mathrm{E}_{#1}\left[#2\right]}

\newcommand{\id}[1]{\mathbb{I}\left(#1\right)}

\def\R{\mathbb{R}}
\def\suchthat{\;:\;}

\title{How Far Can Fairness Constraints Help Recover From Biased Data?}

\author{Mohit Sharma$^1$\thanks{Work done during internship at Microsoft Research India. The author is supported by Microsoft Research India Joint PhD fellowship.} \and Amit Deshpande$^2$}
\date{%
    $^1$IIIT Delhi\\%
    $^2$Microsoft Research India\\[2ex]%
    \today
}

\begin{document}

\maketitle

\begin{abstract}

A general belief in fair classification is that fairness constraints incur a trade-off with accuracy, which biased data may worsen. Contrary to this belief, Blum \& Stangl \cite{blum2019recovering} show that fair classification with equal opportunity constraints even on extremely biased data can recover optimally accurate and fair classifiers on the original data distribution. Their result is interesting because it demonstrates that fairness constraints can implicitly rectify data bias and simultaneously overcome a perceived fairness-accuracy trade-off. Their data bias model simulates under-representation and label bias in underprivileged population, and they show the above result on a stylized data distribution with i.i.d.~label noise, under simple conditions on the data distribution and bias parameters.

We propose a general approach to extend the result of Blum \& Stangl \cite{blum2019recovering} to different fairness constraints, data bias models, data distributions, and hypothesis classes. We strengthen their result, and extend it to the case when their stylized distribution has labels with Massart noise instead of i.i.d.~noise. We prove a similar recovery result for arbitrary data distributions using fair reject option classifiers. We further generalize it to arbitrary data distributions and arbitrary hypothesis classes, i.e., we prove that for any data distribution, if the optimally accurate classifier in a given hypothesis class is fair and \emph{robust}, then it can be recovered through fair classification with equal opportunity constraints on the biased distribution whenever the bias parameters satisfy certain simple conditions. Finally, we show applications of our technique to time-varying data bias in classification and fair machine learning pipelines.
\end{abstract}

\section{Introduction}

Fairness in machine learning has been an important qualitative and quantitative research topic for more than a decade \cite{barocas-hardt-narayanan,mehrabi2021survey}. Several quantitative fairness metrics and various bias mitigation techniques developed over the years are available to practitioners as open-source fairness toolkits \cite{aif360-oct-2018}. Fairness-accuracy trade-offs often seem inevitable, and bias mitigation guarantees can fail when there is a mismatch between train and test data distributions; it is a double whammy when we get neither the desired accuracy nor the desired fairness on deployment. Hence, it is important to understand the role of fair classification and fairness constraints when the training data in machine learning pipelines has systematic biases \cite{blum2019recovering,konstantinov2022impossibility, schrouff2022maintaining}.

Blum \& Stangl \cite{blum2019recovering} propose a model for data bias that simulates systematic under-representation of the underprivileged group and label bias/flip on the underprivileged positive population. 
Applying this data bias model to a stylized data distribution with i.i.d.~label noise, they prove a recovery result that makes the following high-level point: careful choice of fairness constraints can implicitly rectify even extreme data bias and overcome a perceived fairness-accuracy trade-off. Formally, for their stylized distribution with i.i.d.~label noise, they prove that the optimal fair classifier satisfying equal opportunity on a biased (or bias-induced) data distribution coincides with the Bayes optimal classifier (that also happens to be perfectly fair) on the original data distribution, if the bias parameters satisfy certain simple conditions. Their proof is tailored to equal opportunity constraints and uses a careful case analysis and iterative argument that is not amenable to arbitrary distributions or hypothesis classes. 
They also show a specific distribution where a similar result fails to hold if we use demographic parity constraints instead of equal opportunity constraints. 

We take a significantly different approach and observe that the regression function for group-aware classification (i.e., the positive class probability conditioned on the input features and sensitive attribute) undergoes a linear fractional transformation when we apply various data bias models, including the one proposed by Blum \& Stangl \cite{blum2019recovering}. This observation plays an important role in our proofs because the optimal group-aware fair classifier among the hypothesis class of all binary classifiers can be mathematically characterized by group-dependent thresholds on the regression function \cite{menon2018cost,chzhen2019leveraging}. Under the conditions on the data distribution and bias parameters as in Blum \& Stangl \cite{blum2019recovering}, we show that the fairness-corrected thresholds applied to a linear fractional transformation recover the Bayes optimal classifier on the stylized distribution of Blum \& Stangl \cite{blum2019recovering} with i.i.d.~label noise. Our proof uncovers a stronger version of their result that these conditions on the bias parameters are both necessary and sufficient. Moreover, our proof readily extends to the case when the above distribution has Massart or malicious label noise \cite{massart2006risk,rivest1994formal,sloan1996} instead of i.i.d.~label noise. We extend our recovery result to arbitrary data distributions by considering reject option classifiers \cite{bartlett2008classification,cortes2016learning} that can abstain from prediction on a small fraction of inputs by paying a penalty. We further generalize our results to allow arbitrary data distribution as well as arbitrary hypothesis class. Generalizing our ideas beyond threshold-based arguments, we show that on arbitrary data distributions, if the optimal classifier in a given arbitrary hypothesis class is fair and \emph{robust}, then it can be recovered through fair classification on the biased distribution, whenever the bias parameters satisfy certain simple conditions. Finally, we propose multi-step models to capture time-varying data bias and machine learning pipeline, and investigate necessary conditions on the data distribution and bias parameters for similar recovery results in the finite and infinite time horizons. Now we outline the organization of our results in the paper.
\begin{itemize}
\item Section \ref{sec:prelim} contains our theoretical setup for data bias and group-aware fair classification. In Subsection \ref{subsec:data_bias_models}, we describe some recently studied data bias models in fair classification \cite{blum2019recovering,dai2020label,wang2021fair, biswas2019quantifying} (see Examples \ref{eg:blumStangl}, \ref{eg:wang_model}, \& \ref{eg:biswas_model}), and show that all of them result in a linear fractional transformation of the regression function. Subsection \ref{subsec:biased-TPR-TNR} captures how fairness constraints and threshold-based characterizations of optimal fair classifiers change under data bias. We focus on equal opportunity constraints and the data bias model of Blum \& Stangl \cite{blum2019recovering} (Examples \ref{eg:blumStangl}) as an illustrative example running through the rest of our paper.
\item In Section \ref{sec:blumStanglMassart}, we extend the result of Blum \& Stangl \cite{blum2019recovering} to the case when their stylized distribution has Massart label noise instead of i.i.d.~label noise (Theorem \ref{thm:blum_stangl_eo_massart}). We strengthen their result (Theorem \ref{thm:blumStangl-reprove}) and prove its analog for demographic parity constraints replacing equal opportunity (Theorem \ref{thm:dp-recovery}).
\item In Section \ref{sec:reject_option}, we show that our proof technique based on threshold classifiers extends to \emph{arbitrary} data distributions, if we allow reject option classifiers that can abstain from prediction on a small fraction of inputs.
\item In Section \ref{sec:robust}, we invent clever workarounds to generalize our results further to recover the optimal fair and \emph{robust} hypothesis in an \emph{arbitrary} hypothesis class simply by fair classification on the biased version of an \emph{arbitrary} data distribution (Theorem \ref{thm:eo-robust-recovery}).
\item Finally, in Section \ref{sec:time_varying}, we propose time-varying data bias models (also applicable to multi-stage machine learning pipelines) and investigate necessary condition for extending our recovery results above to the finite and infinite time horizon (Theorems \ref{thm:uniform-time-varying-eo} \& \ref{thm:eo-time-varying}). 
\end{itemize}

\section{Related Work} \label{sec:related}
There has been a plethora of recent work on fairness-accuracy trade-offs and data bias \cite{menon2018cost, wick2019unlocking, blum2019recovering, dutta2020there, maity2021does}, but the closest to our work is the result of Blum \& Stangl \cite{blum2019recovering} that we strengthen and generalize in many ways. Though we take equal opportunity constraints \cite{hardt2016} and the data bias model of Blum \& Stangl \cite{blum2019recovering} for under-representation and label bias as an illustrative example running through our paper, our techniques readily extend to other popular fairness constraints such as demographic parity \cite{dwork2012} and other recent data bias models \cite{dai2020label,wang2021fair,biswas2021ensuring}; many possible extensions are covered in the Appendix. Our proof techniques in Section \ref{sec:blumStanglMassart} \& \ref{sec:reject_option} lean heavily on threshold-based characterizations of optimal fair classifiers known in previous work \cite{menon2018cost,chzhen2019leveraging,zeng2022bayes,zeng2022fair}.


Now we describe recent related works that complement our approach to rectify data bias in fair classification. The feasibility of fair classification under data corruption and malicious noise in training data has been studied in \cite{konstantinov2022impossibility, blum2023vulnerability}. Recent work has studied fair classification with noisy sensitive attributes \cite{lamy2019noise,ghosh2023fair,celis2021fair}, noisy labels \cite{fogliato2020fairness}, feature-dependent label bias \cite{jiang2020identifying}, sample selection bias \cite{du2021fair,zhu2023consistent}, subpopulation shift \cite{maity2021does}, and causal models of data bias \cite{plecko2022causal, madras2019fairness, cheong2023causal}. All of them propose algorithmic modifications to the vanilla fair classification to rectify noisy or biased data. 
Complementing the theoretical aspects, recent work has also empirically investigated the effect of data bias and choice of fairness constraints on the accuracy and fairness of various fair classifiers \cite{islam2022through,akpinar2022sandbox,sharma2023testing,ghosh2023fair}.

\section{Data Bias Models \& Fair Classification} \label{sec:prelim}
Let $(X, A, Y)$ be a random data point from the joint distribution $D$ over $\mathcal{X} \times \mathcal{A} \times \mathcal{Y}$, where $\mathcal{X}, \mathcal{A}, \mathcal{Y}$ denote the set of features, the set of sensitive attributes, and the set of class labels, respectively. We assume the feature space $\mathcal{X}$ to be discrete. We consider group-aware classifiers $h: \mathcal{X} \times \mathcal{A} \rightarrow \mathcal{Y}$ and assume, for simplicity, binary sensitive attributes $\mathcal{A} = \{0, 1\}$ and binary class labels $\mathcal{Y} = \{0, 1\}$. The binary classifier of maximum accuracy $h^{*} = \mathrm{argmax}_{h} \prob{h(X, A)=Y}$ is known as the Bayes optimal classifier, and is given by $h^{*}(x, a) = \id{\eta(x, a) \geq 1/2}$, where $\eta(x, a) = \prob{Y=1|X=x, A=a}$ \cite{zeng2022bayes}. This function $\eta$ is known as the \emph{regression function} in statistical machine learning; see Chapter 2 of \cite{devroye1996}. Binary classifiers that apply a threshold on the regression function $\eta$ play a key role in our work as well as other recent works on classification beyond accuracy \cite{elkan2001foundations,singh2022optimal} and fair classification \cite{menon2018cost,chzhen2019leveraging,zeng2022fair}. 
%
%

We use $\tilde{D}$ to denote the biased joint distribution over features, sensitive attributes and class labels, and use $(\tilde{X}, \tilde{A}, \tilde{Y})$ to denote a random data point from the biased distribution $\tilde{D}$. If $X$ (or $A$) in the joint distribution remains unchanged when we go from $D$ to $\tilde{D}$, then we simply use $X$ (or $A$) in the place of $\tilde{X}$ (or $\tilde{A}$).

\subsection{Regression Functions on Biased Data} \label{subsec:data_bias_models}
Below we list some data bias models from recent works on fair classification from biased data \cite{blum2019recovering,dai2020label,biswas2021ensuring}. We make a key observation that for all of these data bias models, the regression function on the biased distribution $\tilde{\eta}(x, a) = \prob{\tilde{Y}=1|\tilde{X}=x, \tilde{A}=a}$ can be expressed as a linear fractional transformation of the regression function on the original distribution $\eta(x, a) = \prob{Y=1|X=x, A=a}$. In other words,
\[
\tilde{\eta}(x, a) = \frac{P \eta(x, a) + Q}{R \eta(x, a) + S}, \quad \text{for some $P, Q, R, S \in \R$}.
\]
%
Please refer to Propositions \ref{prop:blum_stangl_example}, \ref{prop:wang_example} and \ref{prop:biswas_example} in Appendix \ref{appndx: lf_example_transform}.
\begin{example} \label{eg:blumStangl}
\cite{blum2019recovering} Consider a biased distribution $\tilde{D}$ obtained from the original distribution $D$ by the following process defined by under-representation and label bias parameters $\beta_{p}, \beta_{n}, \nu \in (0, 1)$. A random data point $(X, A, Y)$ from $D$ with $A=1$ remains unchanged, the points with $A=0, Y=1$ survive independently with probability $\beta_{p}$, and the points with $A=0, Y=0$ survive independently with probability $\beta_{n}$. Finally, the survived points with $A=0, Y=1$ keep their class label $1$ with probability $1-\nu$, and it gets flipped to $0$ with probability $\nu$. For the privileged group $A=1$, we have $\tilde{\eta}(x, 1) = \eta(x, 1)$, for all $x \in \mathcal{X}$, whereas for the underprivileged group $A=0$, we prove that
~$\tilde{\eta}(x, 0) = \dfrac{(1-\nu) \eta(x, 0)}{(1 - c)\eta(x, 0) + c}$, where $c = \dfrac{\beta_{n}}{\beta_{p}}$ in Proposition \ref{prop:blum_stangl_example}.
\end{example}

\begin{example} \label{eg:wang_model}
\cite{dai2020label, wang2021fair} Consider a biased distribution $\tilde{D}$ obtained from the original distribution $D$ by introducing a group-dependent label flip from $(X, A, Y)$ to $(X, A, \tilde{Y})$ where $\epsilon_{a}^{1} = \prob{\tilde{Y}=0 | Y=1, A=a}$ and $\epsilon_{a}^{0} = \prob{\tilde{Y}=1 | Y=0, A=a}$, with $0 \leq \epsilon_{a}^{1} + \epsilon_{a}^{0} < 1$. 
For this data bias model, we show that $\tilde{\eta}(x,a) = (1 - \epsilon^{1}_{a} - \epsilon^{0}_{a})\eta(x,a) + \epsilon^{0}_{a}$ in Proposition \ref{prop:wang_example}.
\end{example}

\begin{example} \label{eg:biswas_model}
\cite{biswas2021ensuring} Consider a biased distribution $\tilde{D}$ obtained from the original distribution $D$ by introducing a group-dependent prior probability shift such that $\tilde{A} = A$ and $\prob{\tilde{X}=x|\tilde{Y}=i, A=a}=\prob{X=x|Y=i, A=a}$, for any $i, a \in \{0, 1\}$, but $\prob{\tilde{Y}=i|A=a} \neq \prob{Y=i|A=a}$. For this data bias model, we prove that
~$\tilde{\eta}(x, a) = \dfrac{\eta(x, a)}{(1-\alpha) \eta(x, a) + \alpha}$, where $\alpha = \dfrac{\prob{\tilde{Y}=0|A=a} \prob{Y=1|A=a}}{\prob{\tilde{Y}=1|A=a} \prob{Y=0|A=a}}$
in Proposition \ref{prop:biswas_example}.
\end{example}
%
For classifiers that apply a threshold on $\eta(x, a)$ or $\tilde{\eta}(x, a)$, it is important to understand when the above linear fractional transformations are order-preserving.
\begin{proposition} \label{prop:order-preserving}
Suppose $S \geq 0$, $R+S \geq 0$, and $PS - QR \geq 0$, then the transformation $\tilde{\eta}(x, a) = \dfrac{P \eta(x, a) + Q}{R \eta(x, a) + S}$ is order-preserving, i.e., $\eta(x_{1}, a) \leq \eta(x_{2}, a)$ iff $\tilde{\eta}(x_{1}, a) \leq \tilde{\eta}(x_{2}, a)$.  
\end{proposition}
The proof is provided in Appendix \ref{appndx:proof_order}. Note that all of the above data bias models satisfy the order-preservation property. For the rest of the paper, we exclusively focus on the under-representation and label bias model in Example \ref{eg:blumStangl} \cite{blum2019recovering} as an illustrative example. Our techniques are flexible and can be applied to obtain similar results for other data bias models.

\subsection{Fair Classification on Biased Data} \label{subsec:biased-TPR-TNR}
Demographic Parity (equal group-wise positivity rates) and Equal Opportunity (equal group-wise true positive rates) are two most popular fairness constraints in classification. It is easy to see that the true positive rates (TPRs) and the true negative rates (TNRs) for a classifier on the biased data distribution can be expressed as linear combinations of its TPRs and TNRs on the original data distribution.

\begin{proposition} \label{prop:tilde-tpr-tnr}
Let $D$ be any distribution on $\mathcal{X} \times \mathcal{A} \times \mathcal{Y}$ and let $\tilde{D}$ be its biased version defined using any of the data bias models defined in Subsection \ref{subsec:data_bias_models}. Let $h$ be any hypothesis, and let $TPR_{a}(h)$ and $\widetilde{TPR}_{a}(h)$ be its true positive rates according to $D$ and $\tilde{D}$, respectively, conditioned on the underprivileged group $A=a$. Let $TNR_{a}(h)$ and $\widetilde{TNR}_{a}(h)$ be the true negative rates defined similarly. Then
\begin{enumerate}
    \item $\widetilde{TPR}_{a}(h) = \prob{Y=1 | \tilde{Y}=1, A=a}~ TPR_{a}(h) + \prob{Y=0 | \tilde{Y}=1, A=a}~ FPR_{a}(h)$,~ and
    \item $\widetilde{TNR}_{a}(h) = \prob{Y=0 | \tilde{Y}=0, A=a}~ TNR_{a}(h) + \prob{Y=1 | \tilde{Y}=0, A=a}~ FNR_{a}(h).$
\end{enumerate}
\end{proposition}




The proof of Proposition \ref{prop:tilde-tpr-tnr} is given in Appendix \ref{appndx:tpr_tnr_proofs}, and we get the following interesting corollary for the data bias model in Example \ref{eg:blumStangl} \cite{blum2019recovering}.
%
\begin{corollary} \label{corr:tpr_tnr_sp}
Let $D$ be any distribution and $\tilde{D}$ be its biased version as in Example \ref{eg:blumStangl}. Given any hypothesis class $\mathcal{H}$, let $\mathcal{H}_{\text{fair, EO}}$ be its subset that satisfies equal opportunity on the original distribution $D$, i.e., $\mathcal{H}_{\text{fair, EO}} = \{f \in \mathcal{H} \suchthat TPR_{0}(f) = TPR_{1}(f)\}$.
Similarly, let $\widetilde{\mathcal{H}}_{\text{fair, EO}} = \{f \in \mathcal{H} \suchthat \widetilde{TPR}_{0}(f) = \widetilde{TPR}_{1}(f)\}$ be the subset of $\mathcal{H}$ satisfying equal opportunity on the biased distribution $\tilde{D}$. 
%
Then $\widetilde{TPR}_{0}(h) = TPR_{0}(h)$ and $\widetilde{TPR}_{1}(h) = TPR_{1}(h)$, and hence, $\mathcal{H}_{\text{fair, EO}} = \widetilde{\mathcal{H}}_{\text{fair, EO}}$.
\end{corollary}
\begin{remark}
Corollary \ref{corr:tpr_tnr_sp} can be easily extended to other fairness metrics such as demographic parity as well as the hypothesis class of \emph{approximately} fair classifiers that satisfy fairness constraints up to a small additive or multiplicative error. However, we focus on exact equal opportunity as in Blum \& Stangl \cite{blum2019recovering} for a direct, illustrative application. Please see Appendix \ref{appndx:tpr_tnr_proofs} for additional results.
\end{remark}

The optimal fair classifier for equal opportunity (similarly, demographic parity) on a given data distribution can be expressed by group-dependent thresholds applied to the regression function \cite{menon2018cost, chzhen2019leveraging}. Since the regression function $\tilde{\eta}(x, a)$ on the biased distribution $\tilde{D}$ is an order-preserving linear fractional transformation of $\eta(x, a)$, the optimal fair classifier for equal opportunity on $\tilde{D}$ is equivalent to applying group-dependent thresholds to $\eta(x, a)$. 
\begin{proposition} \label{prop:biased-eo-threshold}
For any distribution $D$ and its biased version $\tilde{D}$ described in Example \ref{eg:blumStangl}, let $\tilde{h}_{EO}$ be a classifier of the maximum accuracy among all binary classifiers that satisfy equal opportunity on $\tilde{D}$. Then there exists $\lambda^{*} \in \R$ such that $\tilde{h}_{EO}(x, a) = \id{\eta(x, a) \geq t_{a}}$, where
\begin{align*}
t_{a} = \begin{cases} \dfrac{1}{1 + \dfrac{1-2\nu}{c} + \dfrac{\lambda^{*}}{\beta_{n} \prob{Y=1, A=0}}}, & \text{for $a=0$} \\
\dfrac{1}{2 - \dfrac{\lambda^{*}}{\prob{Y=1, A=1}}}, & \text{for $a=1$}.
\end{cases}
\end{align*}
\end{proposition}
The Proof of Proposition is given in Appendix \ref{appndx:tpr_tnr_proofs}. We can similarly derive the optimal threshold with the biased distribution for the Demographic parity constraint (Proposition \ref{prop:biased-dp-threshold} in Appendix \ref{appndx:tpr_tnr_proofs}).

\section{Recovering Optimal Classifier from Biased Data for Massart Label Noise} \label{sec:blumStanglMassart}
Blum \& Stangl \cite{blum2019recovering} consider a stylized distribution $D$ with i.i.d.~label noise and show that the optimal fair classifier $\tilde{h}_{EO}$ on the biased distribution $\tilde{D}$ (defined in Example \ref{eg:blumStangl}) recovers the Bayes optimal (and fair) classifier $h^{*}$ on the original distribution $D$, if the bias parameters satisfy certain simple conditions. Note that this does not require knowing, estimating, or correcting for data bias explicitly, and their result holds even for extreme under-representation and label bias in $\tilde{D}$. We first demonstrate the utility of our technique by generalizing the recovery result of Blum \& Stangl \cite{blum2019recovering} to the case of Massart noise \cite{massart2006risk}. We describe the distribution setup below, give a sketch of our proof, and point out the generality of our technique compared to Blum \& Stangl \cite{blum2019recovering}.

\subsection{Generalizing Blum \& Stangl \cite{blum2019recovering} Recovery Result for Massart Noise} \label{subsec:massart}
Assume any arbitrary data distribution $D$ on $\mathcal{X} \times \mathcal{A}$. Let $\prob{A=0}=r$ and $\prob{A=1}=1-r$, for some $0<r<1$. Let $h: \mathcal{X} \times \mathcal{A} \rightarrow \mathcal{Y}$ be any hypothesis that satisfies $\prob{h(X, A)=1|A=0} = \prob{h(X, A)=1|A=1}$. Let $\delta < 1/2$, and extend the distribution to $\mathcal{X} \times \mathcal{A} \times \mathcal{Y}$ as follows. $Y|X=x, A=a$ takes value $h(x, a)$ with probability $1 - \delta(x, a)$, and $\neg h(x, a)$ with probability $\delta(x, a)$, for some $\delta(x, a) \leq \delta$. Let $D$ be the resulting distribution on $\mathcal{X} \times \mathcal{A} \times \mathcal{Y}$. This type of bounded noise in class label is popularly known as Massart noise\footnote{Equivalently known as malicious classification noise in previous work \cite{rivest1994formal,sloan1996}.} in literature, based on a noise model proposed by Massart and N\'{e}d\'{e}lec\cite{massart2006risk}. We assume that the Massart noise is added in a way that equalizes the base rates on the two protected groups, i.e., $ Pr(Y=1|A=0) = Pr(Y=1|A=1) = q$. Since $\delta < 1/2$, the Bayes optimal classifier $h^{*}$ on the distribution $D$ coincides with $h$ and satisfies Equal Opportunity.
%
\begin{theorem} \label{thm:blum_stangl_eo_massart}
For any distribution $D$ defined as above and its biased version $\tilde{D}$ defined as in Example \ref{eg:blumStangl} using the bias parameters $\beta_{p}, \beta_{n}, \nu \in (0, 1)$. If the data distribution and bias parameters satisfy
\begin{align*}
    (1-r)(1-2\delta) + r\left((1-\delta) \beta_{p} (1-2\nu) - \delta \beta_{n}\right) > 0 \\
    \text{ and }  \quad \quad \quad \quad \quad \quad \quad \quad \quad \quad \\
    (1-r)(1-2\delta) + r\left((1-\delta) \beta_{n} - \delta \beta_{p} (1-2\nu)\right) > 0,
\end{align*}
then the optimal equal opportunity classifier on the biased distribution $\tilde{D}$ recovers the Bayes optimal classifier on the original distribution $D$, i.e., $\tilde{h}_{EO} \equiv h^{*}$.
\end{theorem}
The proof of Theorem \ref{thm:blum_stangl_eo_massart} is given in Appendix \ref{appndx:blumStanglProofs}. The same proof also works for group-dependent Massart noise, i.e., there exist $\delta_{0}, \delta_{1} < 1/2$ such that $\delta(x, a) \leq \delta_{a}$, for all $(x, a) \in \mathcal{X} \times \mathcal{A}$.

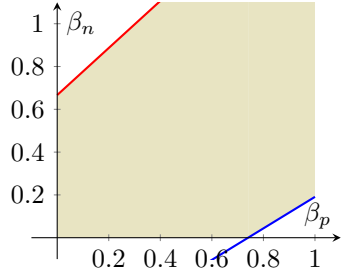
\begin{figure}[!h]
\centering
\begin{tikzpicture}[scale=0.6]
    \begin{axis}[
    axis line style={->},
    axis lines=middle,
    xlabel=$\beta_{p}$,
    ylabel=$\beta_{n}$,
    xmin=-0.1, xmax=1.1,
    ymin=-0.1, ymax=1.1,
    axis on top=true,
    domain=0:1,
    ]
    \addplot[draw=red,thick,name path=A] {(0.495*x+0.3)/0.45}; 
    \addplot[draw=blue,thick,name path=B] {(0.405*x-0.3)/0.55}; 
    \addplot[name path=C, draw=none] {0};
    \addplot[olive!20] fill between[of=A and C,soft clip={domain=0:0.3/0.405}];
    \addplot[olive!20] fill between[of=A and B,soft clip={domain=0.3/0.405:1}];
    \end{axis}
\end{tikzpicture}
\vspace{-.5cm}
\caption{Recovery region for $\beta_{p}, \beta_{n} \in (0, 1]$ given by the constraints $(1-r)(1-2\delta) + r\left((1-\delta) \beta_{p} (1-2\nu) - \delta \beta_{n}\right) > 0$ and $(1-r)(1-2\delta) + r\left((1-\delta) \beta_{n} - \delta \beta_{p} (1-2\nu)\right) > 0$ as in Theorem \ref{thm:blum_stangl_eo_massart}, when $r=0.25$, $\nu=0.05$, and $\delta=0.45$. We can recover optimal and fair classifiers for a large range of data biases, including extreme under-representation, i.e., region close to the origin $(0, 0)$, by applying just equal opportunity constraints.}
\label{fig:blumStanglMassart}
\end{figure}

The conditions in Theorem \ref{thm:blum_stangl_eo_massart} above are identical to those in the recovery result of Blum \& Stangl \cite{blum2019recovering} (Theorem 4.1 in their paper) that only works for the special case when $\delta(x, a) = \delta$, for all $(x, a) \in \mathcal{X} \times \mathcal{A}$. Their proof is arguably less flexible to other models of label noise and data bias, as it relies on clever, iterative modifications of an initial fair classifier until its accuracy cannot be improved further. For their special case $\delta(x, a) = \delta$, for all $(x, a) \in \mathcal{X} \times \mathcal{A}$, our technique gives a stronger statement that the conditions in Theorem \ref{thm:blum_stangl_eo_massart} are in fact both necessary and sufficient. Figure \ref{fig:blumStanglMassart} illustrates the recovery conditions in Theorem \ref{thm:blum_stangl_eo_massart} for reasonably chosen group proportion parameter $r$, label bias $\nu$ and $\delta$.
\begin{theorem} (a slightly stronger version of Theorem 4.1 in Blum \& Stangl \cite{blum2019recovering}) \label{thm:blumStangl-reprove}
For the data distribution $D$ described above with $\delta(x, a) = \delta$, for all $(x, a) \in \mathcal{X} \times \mathcal{A}$, and its biased version $\tilde{D}$ as described in Example \ref{eg:blumStangl}, the data distribution and bias parameters satisfy
\begin{align*}
(1 - r)(1 - 2\delta) + r((1 - \delta)\beta_{p}(1 - 2\nu) - \delta\beta_{n}) > 0 \\ 
\text{ and } \quad \quad \quad \quad \quad \quad \quad \quad \quad \quad\\
(1 - r)(1 - 2\delta) + r((1 - \delta)\beta_{n} - (1 - 2\nu)\delta\beta_{p}) > 0,
\end{align*}
if and only if the optimal equal opportunity classifier on $\tilde{D}$ recovers the Bayes optimal classifier on $D$, i.e., $\tilde{h}_{EO} \equiv h_{EO} \equiv h^{*}$.
\end{theorem}

Similarly, we can also obtain necessary and sufficient conditions for when the optimal demographic parity classifier on $\tilde{D}$ recovers $h^{*}$. Blum \& Stangl \cite{blum2019recovering} only give a specific example where such a recovery is impossible via demographic parity constraints (see Subsection 3.1 of \cite{blum2019recovering}) but do not prove any analog of Theorem \ref{thm:blumStangl-reprove} (see Table 1 in \cite{blum2019recovering}).

\begin{theorem} \label{thm:dp-recovery}
For the data distribution $D$ described above with $\delta(x, a) = \delta$, for all $(x, a) \in \mathcal{X} \times \mathcal{A}$, and its biased version $\tilde{D}$ as described in Example \ref{eg:blumStangl}, the data distribution and bias parameters satisfy $\beta_{p}(1 - \delta)(1 - 2\delta - 2r(\nu - \delta)) + \delta\beta_{n}(1 - 2\delta -2r(1 - \delta)) > 0$ and $\beta_{p}\delta(1 - 2r(1 - \nu) - 2\delta(1 - r)) + (1 - \delta)\beta_{n}(1 - 2\delta(1 - r)) > 0$,
if and only if the optimal demographic parity classifier on $\tilde{D}$ recovers the Bayes optimal classifier on $D$, i.e., $\tilde{h}_{DP} \equiv h_{DP} \equiv h^{*}$.
\end{theorem}
Appendix \ref{appndx:blumStanglProofs} contains the proofs of Theorems \ref{thm:blumStangl-reprove} \& \ref{thm:dp-recovery}.



\subsection{Proof Sketches}
We briefly outline our proof technique for Theorems \ref{thm:blum_stangl_eo_massart}, \ref{thm:blumStangl-reprove}, \ref{thm:dp-recovery} to explain an important technical contribution of our paper. We write fairness constrained accuracy maximization using a Lagrange multiplier $\lambda$. Proposition \ref{prop:biased-eo-threshold} characterizes $\tilde{h}_{EO}$ as $\tilde{h}_{EO}(x, a) = \id{\eta(x, a) \geq t_{a}}$ that applies group-dependent thresholds $t_{a}$ on $\eta(x, a)$, where the threshold $t_{a}$ is actually a function of the optimal Lagrange multiplier $\lambda^{*}$, the data distribution parameters, and the bias parameters. We show (see Lemma \ref{thm:eta-lemma}) that as long as these thresholds $t_{a}$ applied to $\eta(x, a)$ for both the groups $a=0$ and $a=1$ lie within the interval $(\delta, 1-\delta)$, we have $\tilde{h}_{EO} \equiv h^{*}$. We show that the possible choices of $\lambda^{*}$ are narrowed down to allow only $\tilde{h}_{EO} \equiv h^{*}$ using the given conditions on the data distribution and bias parameters, and the fairness constraint on the resulting threshold classifier. We prove that the conditions in Theorem \ref{thm:blumStangl-reprove} are necessary and sufficient for the optimal $\lambda^{*}$ parameter in Proposition \ref{prop:biased-eo-threshold} to satisfy that the group-dependent thresholds $t_{0}$ and $t_{1}$ lie in the interval $(\delta, 1-\delta)$, and equivalently, $\tilde{h}_{EO} \equiv h^{*}$.


\section{Recovery of Optimal Reject Option Classifiers from Biased Data for Arbitrary Data Distributions} \label{sec:reject_option}

A major limitation of Theorems \ref{thm:blum_stangl_eo_massart} and \ref{thm:blumStangl-reprove} is that they work only on stylized distributions, where the label noise is either i.i.d. or Massart. In this section, we remove this limitation by proving a similar recovery result for arbitrary data distributions. Massart or i.i.d. label noise creates a clear separation between $\eta(x, a)$ values (or high-risk and low-risk), and allows a small interval margin for the group-wise thresholds to recover $h^{*}$. To mimic this in an arbitrary data distribution, we consider \emph{reject option classifiers} that are allowed to abstain from prediction by paying a penalty \cite{bartlett2008classification, cortes2016learning, charoenphakdee2021classification, schreuder2021classification, franc2023optimal}.
Models that abstain from prediction play an important role in responsible machine learning, as predictions of high uncertainty can be overseen by a human-in-the-loop. As a result, many recent papers have studied reject option classifiers for fair classification \cite{madras2018predict, schreuder2021classification, shah2022selective}.

Let $D$ be an arbitrary distribution on $\mathcal{X} \times \mathcal{A} \times \mathcal{Y}$, with $\mathcal{A} = \{0, 1\}$ and $\mathcal{Y} = \{0, 1\}$. A reject option classifier $g: \mathcal{X} \times \mathcal{A} \rightarrow \{0, 1, \bot\}$ either rejects or abstains from prediction on an input $(x, a)$, denoted by $g(x, a) = \bot$, or it predicts $g(x, a) = h(x, a)$ using a binary classifier $h: \mathcal{X} \times \mathcal{A} \rightarrow \{0, 1\}$. Let $\mathcal{H}$ denote the hypothesis class of all binary classifiers $h: \mathcal{X} \times \mathcal{A} \rightarrow \{0, 1\}$, and let $\mathcal{H}^{\text{rej}}$ be the hypothesis class of all $g: \mathcal{X} \times \mathcal{A} \rightarrow \{0, 1, \bot\}$. For rejection penalty given by $\delta > 0$, the optimal reject option classifier is defined as 
\begin{align*}
h^{\text{rej}} & = \underset{g \in \mathcal{H}^{\text{rej}}}{\mathrm{argmin}}~ \prob{g(X, A) \neq Y, g(X, A) \neq \bot} \\
& \qquad \qquad \qquad \qquad + \delta~ \prob{g(X, A) = \bot}.
\end{align*}
Proposition \ref{prop:opt-rej} characterizes the optimal reject option classifier on any distribution $D$, which is a generalization of the known forms of Optimal reject option classifiers (Section $1$ in \cite{bartlett2008classification}, Section $2$ in \cite{cortes2016learning}).
\begin{proposition} \label{prop:opt-rej}
Let $D$ be any distribution on $\mathcal{X} \times \mathcal{A} \times \mathcal{Y}$, with $\mathcal{A} = \{0, 1\}$ and $\mathcal{Y} = \{0, 1\}$. Let $\delta \in [0, 1/2)$ denote the rejection penalty and $h^{\text{rej}}$ be the optimal reject option classifier defined as above. Then
\[
h^{\text{rej}}(x, a) = \begin{cases} 0, \quad \text{if $\eta(x, a) \leq \delta$} \\
\bot, \quad \text{if $\eta(x, a) \in (\delta, 1-\delta)$} \\
1, \quad \text{if $\eta(x, a) \geq 1-\delta$},
\end{cases}
\]
where $\eta(x, a) = \prob{Y=1|X=x, A=a}$.
\end{proposition}

The proof of Proposition \ref{prop:opt-rej} is given in Appendix \ref{appndx:rejectionProofs}. Proposition \ref{prop:opt-rej} shows how reject option induces a separation between high and low $\eta(x, a)$ values similar to the case of i.i.d. or Massart label noise. A larger separation has an obvious trade-off with a larger fraction of inputs being turned away for model prediction.

Now assume that the optimal reject option classifier $h^{\text{rej}}$ satisfies equal opportunity on the non-rejected part of the distribution $D$, i.e., $\prob{h^{\text{rej}}(X, A)=1|Y=1, A=0} = \prob{h^{\text{rej}}(X, A)=1|Y=1, A=1}$. Theorem \ref{thm:reject_recovery} shows that the optimal equal opportunity classifier on the biased distribution $\tilde{D}$ recovers $h^{\text{rej}}$ on the non-rejected inputs, if the data distribution and bias parameters satisfy the same conditions as in Theorems \ref{thm:blum_stangl_eo_massart} \& \ref{thm:blumStangl-reprove}.
\begin{theorem} \label{thm:reject_recovery}
Let $D$ be an arbitrary distribution on $\mathcal{X} \times \mathcal{A} \times \mathcal{Y}$, with $\mathcal{A} = \{0, 1\}$ and $\mathcal{Y} = \{0, 1\}$. Let $\delta \in [0, 1/2)$ be the rejection penalty and suppose the optimal reject option classifier $h^{\text{rej}}$ defined above satisfies equal opportunity on the non-rejected part of distribution $D$. Let $\tilde{D}$ be a biased version of $D$ defined as in Example \ref{eg:blumStangl} using bias parameters $\beta_{p}, \beta_{n}, \nu \in (0, 1)$, and let $\tilde{h}_{EO}$ be the optimal equal opportunity classifier on $\tilde{D}$. If the data distribution and bias parameters satisfy
\begin{align*}
    (1-r)(1-2\delta) + r\left((1-\delta) \beta_{p} (1-2\nu) - \delta \beta_{n}\right) > 0 \\
    \text{ and }  \quad \quad \quad \quad \quad \quad \quad \quad \quad \quad \\
    (1-r)(1-2\delta) + r\left((1-\delta) \beta_{n} - \delta \beta_{p} (1-2\nu)\right) > 0,
\end{align*}
then $\tilde{h}_{EO}(x, a) = h^{\text{rej}}(x, a)$ whenever $h^{\text{rej}}(x, a) \neq \bot$. 
\end{theorem}
A complete proof of Theorem \ref{thm:reject_recovery} can be found in Appendix \ref{appndx:rejectionProofs}. Note that if we consider the entire distribution $D$ instead of only the non-rejected inputs, then $\prob{\tilde{h}_{EO}(X, A) \neq h^{\text{rej}}(X, A)} \leq \prob{h^{\text{rej}}(X, A) = \bot}$. In other words, if $\prob{h^{\text{rej}}(X, A) = \bot}$ is small, then $\tilde{h}_{EO}$ matches $h^{\text{rej}}$ on distribution $D$ with high probability.

\section{Recovering Robust Hypothesis under Data Bias for Arbitrary Data Distributions and Arbitrary Hypothesis Classes} \label{sec:robust}
In this section, we remove the restriction on hypothesis class $\mathcal{H}$, assumed to be the class of all group-aware binary classifiers in Sections \ref{sec:blumStanglMassart} \& \ref{sec:reject_option}. Note that the characterization of optimal fair classifiers using group-aware thresholds on the regression function $\eta(x, a)$ plays an important role in our proofs from Sections \ref{sec:blumStanglMassart} \& \ref{sec:reject_option}. For an arbitrary hypothesis class $\mathcal{H}$, even the classifier $h^{*} \in \mathcal{H}$ that maximizes accuracy on $D$ need not be a threshold classifier on $\eta(x, a)$. To work around this, we make an assumption that the optimal (and fair) classifier that we want to recover under data bias must be \emph{robust} under small perturbations to the data distribution $D$. Our definition of $\epsilon$-robustness is motivated by the linear fractional transformations of regression function observed in various data bias models earlier (see Section \ref{sec:prelim}).



Let $D$ be any distribution on $\mathcal{X} \times \mathcal{A} \times \mathcal{Y}$, with $\mathcal{A} = \{0, 1\}$ and $\mathcal{Y} = \{0, 1\}$. Let $\prob{A=0} = r$, $\prob{A=1} = 1-r$, and let $\prob{Y=1|A=0} = \prob{Y=1|A=1} = q$. Let $h^{*}$ be the most accurate classifier in $\mathcal{H}$, i.e., 
\begin{align*} 
h^{*} & = \underset{h \in \mathcal{H}}{\mathrm{argmax}}~ \prob{h(X, A) = Y} \\
& = \underset{h \in \mathcal{H}}{\mathrm{argmax}}~ \expecto{(X, A)}{h(X, A) (2\eta(X, A)-1)}.
\end{align*} 
Note that, for an arbitrary hypothesis class $\mathcal{H}$, the optimal $h^{*}$ need not be a threshold classifier on $\eta(x, a)$. As in the previous sections, we assume that $h^{*}$ satisfies equal opportunity on the distribution $D$, i.e., $\prob{h^{*}(X, A)=1|Y=1, A=0} = \prob{h^{*}(X, A)=1|Y=1, A=1}$. 

\begin{definition} \label{def:eps-robust-property}
We define $h^{*} \in \mathcal{H}$ to be \emph{$\epsilon$-robust} if, for any distribution $D'$ with random data points $(X, A, Y')$ s.t.
\[
\eta'(x, a) = \prob{Y'=1|X=x, A=a} = \frac{P_{a} \eta(x, a) + Q_{a}}{R_{a} \eta(x, a) + S_{a}},
\]
with $S_{a} = 1, |R_{a}| \leq \epsilon, |Q_{a}| \leq \epsilon, 1-\epsilon \leq P_{a} \leq 1+\epsilon$, and $P_{a}S_{a} - Q_{a}R_{a} \geq 0$, the optimal classifier $h^{*}$ remains unchanged, i.e., $h^{*} = h' = \underset{h \in \mathcal{H}}{\mathrm{argmax}}~ \prob{h(X, A) = Y'}$. 
\end{definition}
The above conditions give a scale-invariant proxy to say that the linear fractional transformation defined by $P_{a}, Q_{a}, R_{a}, S_{a}$ is order-preserving and when it is appropriately scaled to make $S_{a}=1$, it is $\epsilon$-close to the identity transformation. Definition \ref{def:eps-robust-property} says that the classifier $h^{*}$ of maximum accuracy in $\mathcal{H}$ is robust to small near-identity perturbations of the data distribution $D$.

Now we are ready to state our result for recovering robust, fair hypothesis from biased data on arbitrary data distributions and arbitrary hypothesis classes. 
\begin{theorem} \label{thm:eo-robust-recovery}
For any distribution $D$ and any hypothesis class $\mathcal{H}$, if the optimal classifier $h^{*} \in \mathcal{H}$ is $\epsilon$-robust and the bias parameters $\beta_{p}, \beta_{n}, \nu$ satisfy $(1-\epsilon)\beta_{n} \leq \beta_{p} \leq (1+\epsilon)\beta_{n}$,
\begin{align*}
    r \left((1-\nu)\beta_{p} - (1-\epsilon)\beta_{n}\right) + \epsilon (1-r) \geq 0 \\
    \text{ and }  \quad \quad \quad \quad \quad \quad \quad \quad\\
    r \left((1+\epsilon)\beta_{n} - (1-\nu)\beta_{p}\right) + \epsilon (1-r) \geq 0,
\end{align*}
then the optimal equal opportunity classifier from $\mathcal{H}$ on the biased distribution $\tilde{D}$ recovers the optimal classifier from $\mathcal{H}$ on the original distribution $D$, i.e., $\tilde{h}_{EO} \equiv h^{*}$.
\end{theorem}
We prove Theorem \ref{thm:eo-robust-recovery} in Appendix \ref{appndx:proof_eo-robust-recovery}. Our proof reuses the basic characterization of optimal fair classifiers using Lagrange multipliers, and although it uses the class probabilities $\eta(x, a)$'s in a crucial way, it circumvents the need for threshold-based arguments completely. Figure \ref{fig:eps-robust} illustrates the recovery conditions in Theorem \ref{thm:eo-robust-recovery}, for reasonably chosen group proportion parameter $r$, label bias $\nu$ and hypothesis robustness parameter $\epsilon$.

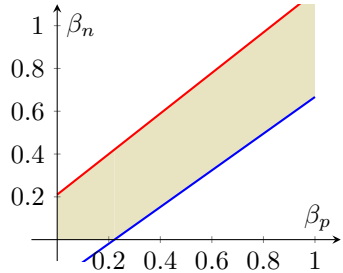
\begin{figure}[!h]
\centering
\begin{tikzpicture}[scale=0.6]
    \begin{axis}[
    axis line style={->},
    axis lines=middle,
    xlabel=$\beta_{p}$,
    ylabel=$\beta_{n}$,
    xmin=-0.1, xmax=1.1,
    ymin=-0.1, ymax=1.1,
    axis on top=true,
    domain=0:1,
    ]
    \addplot[draw=red,thick,name path=A] {(0.9*x+0.2)/0.95}; 
    \addplot[draw=blue,thick,name path=B] {(0.9*x-0.2)/1.05}; 
    \addplot[name path=C, draw=none] {0};
    \addplot[olive!20] fill between[of=A and C,soft clip={domain=0:0.2/0.9}];
    \addplot[olive!20] fill between[of=A and B,soft clip={domain=0.2/0.9:1}];
    \end{axis}
\end{tikzpicture}
\vspace{-.4cm}
\caption{Recovery region for $\beta_{p}, \beta_{n} \in (0, 1]$ given by the constraints $r \left((1-\nu)\beta_{p} - (1-\epsilon)\beta_{n}\right) + \epsilon (1-r) \geq 0$ and $r \left((1+\epsilon)\beta_{n} - (1-\nu)\beta_{p}\right) + \epsilon (1-r) \geq 0$ as in Theorem \ref{thm:eo-robust-recovery}, when $r=0.2$, $\nu=0.1$, and $\epsilon=0.05$. Even for arbitrary distributions and hypothesis classes, optimal and fair classifiers can be recovered from extreme under-representation and for a large range of data biases using just equal opportunity constraints.}
\label{fig:eps-robust}
\end{figure}

\section{Recovering from Time-Varying Data Bias} \label{sec:time_varying}
Data biases arise commonly in machine learning pipelines where data changes for downstream applications and over time. As an application of our techniques, we demonstrate how we can obtain recovery guarantees in data bias pipelines, i.e., when at every time step, we obtain a new shifted distribution. We model time-varying data bias as repeated applications of single-step data bias model (e.g., Example \ref{eg:blumStangl} used in previous sections).

Let $\tilde{D}_{t}$ be the biased data distribution obtained from an original distribution $D$, when the data bias model described in Example \ref{eg:blumStangl} gets applied repeatedly $t$ times with possibly different bias parameters $(\beta_{p,i}, \beta_{n,i}, \nu_{i})$ at $i$-th time step, and let $c_{i} = \dfrac{\beta_{n,i}}{\beta_{p,i}}$. Since the composition of linear fractional transformations remains a linear fractional transformation, we obtain the following generalization for how the regression function changes over time.
\begin{proposition} \label{prop:time-eta-tilde}
Let $(X, A, Y)$ denote a random data point from any given distribution $D$ and let $(X, A, \tilde{Y}_{t})$ be a random data point from its corresponding biased distribution $\tilde{D}_{t}$ after applying the multi-stage time-varying data bias model described above. Let $\eta(x, a) = \prob{Y=1|X=x, A=a}$ and $\tilde{\eta}_{t}(x, a) = \prob{\tilde{Y}_{t}=1|X=x, A=a}$. Then

\[
\tilde{\eta}_{t}(x, 0) = \frac{\eta(x, 0)}{{\displaystyle \sum_{i=1}^{t} \left(\dfrac{1 - c_{i}}{1 - \nu_{i}} 
\prod_{j=i+1}^{t} \dfrac{c_{j}}{1 - \nu_{j}}\right) \eta + \prod_{i=1}^{t} \dfrac{c_{i}}{1 - \nu_{i}}}},
\]
where $\displaystyle \prod_{j=t+1}^{t} \dfrac{c_{j}}{1 - \nu_{j}} \overset{\mathrm{def}}{=} 1$.
\end{proposition}
Note that $\tilde{D}_{t}$ conditioned on $A=1$ remains unchanged, and therefore, $\tilde{\eta}_{t}(x, 1) = \eta(x, 1)$, since we are looking at the data bias model in Example \ref{eg:blumStangl}. The proof of Proposition \ref{prop:time-eta-tilde} is by simple induction on $t$. As a result of Proposition \ref{prop:time-eta-tilde}, $\tilde{\eta}_{t}(x,a)$ can be expressed as another single-step data bias model that directly transforms $\eta(x, a)$ into $\tilde{\eta}_{t}(x, a)$ using a linear fractional transformation with $P=0, Q=0, R = \displaystyle \sum_{i=1}^{t} \left(\dfrac{1 - c_{i}}{1 - \nu_{i}} \displaystyle \prod_{j=i+1}^{t} \dfrac{c_{j}}{1 - \nu_{j}}\right), \text{ and } S = \displaystyle \prod_{i=1}^{t} \dfrac{c_{i}}{1 - \nu_{i}}$. The above $P, Q, R, S$ obey the conditions in Proposition \ref{prop:order-preserving}, so the corresponding linear fractional transformation is order-preserving.
\vspace{-.2cm}
\subsection{Repeated Data Bias \& Infinite Time Horizon} \label{subsec: uniform-time-varying}
As a warm-up, we first study a simpler case where the bias parameters do not change at each time step, i.e., $\beta_{p,i} = \beta_{p}$, $\beta_{n,i} = \beta_{n}$, and $\nu_{i} = \nu$, for all $i \in [t]$). Equivalently, the same data bias model with the same bias parameters gets applied repeatedly $t$ times. We work with the original distribution $D$ described as in Theorem \ref{thm:blumStangl-reprove} for the ease of analysis, and assume that $\prob{A=0} = r$, $\prob{A=1} = 1-r$ and $\prob{Y=1|A=0} = \prob{Y=1|A=1} = q$. First, we show Theorem \ref{thm:uniform-time-varying-eo} about when the optimal equal opportunity classifier $\tilde{h}_{EO,t}$ on the biased distribution $\tilde{D}_{t}$ can recover $h^{*}$ for the infinite time horizon as $t \rightarrow \infty$, and the necessary conditions on the data distribution and bias parameters to allow that. The proof of Theorem \ref{thm:uniform-time-varying-eo} is given in Appendix \ref{appndx:proofs-time-varying}.

\begin{theorem} \label{thm:uniform-time-varying-eo}
Let $\tilde{D}_{t}$ be the biased distribution obtained by applying the bias model in Example \ref{eg:blumStangl} repeatedly $t$ times with $\beta_{p,i} = \beta_{p}$, $\beta_{n,i} = \beta_{n}$, and $\nu_{i} = \nu$, for all $i \in [t]$ on a given distribution $D$ defined as in Theorem \ref{thm:blumStangl-reprove}. Let $h^{*}$ be the Bayes optimal classifier on $D$ and let $\tilde{h}_{EO,t}$ be the optimal equal opportunity classifier on $\tilde{D}_{t}$.
\begin{itemize}
\vspace{-.3cm}
\item If $\beta_{n} < 1$ then as $t \rightarrow \infty$, we have $\eta(x,0) \rightarrow 0$, for all $x \in \mathcal{X}$. Thus, we cannot have $\tilde{h}_{EO,t} \equiv h^{*}$ as $t \rightarrow \infty$.
\vspace{-.2cm}
\item If $\beta_{n} = 1$ and $\tilde{h}_{EO,t} \equiv h^{*}$ as $t \rightarrow \infty$, then the data distribution and bias parameters must satisfy
\vspace{-.2cm}
\[
1 - \dfrac{(1-2\delta)}{(1-\delta) r} < \dfrac{\nu \beta_{p}}{1-\beta_{p}(1 - \nu)} < 1 + \dfrac{(1 - 2\delta)}{\delta r}.
\]
\end{itemize}
\end{theorem}

\subsection{Time-Varying Data Bias Pipeline with Finite Steps} \label{subsec: nonuniform-time-varying}
Now we generalize the necessary conditions in Theorem \ref{thm:blumStangl-reprove} for repeated application of the data bias model from Example \ref{eg:blumStangl}, with possibly different bias parameters at each time step. We assume that the bias parameters can vary but are bounded.

\begin{theorem} \label{thm:eo-time-varying}
Let $D$ be a data distribution described as in Theorem \ref{thm:blumStangl-reprove} and $\tilde{D}_{t}$ be the resulting biased distribution after repeated application of the data bias model from Example \ref{eg:blumStangl} to $D$ in $t$ steps, with bounded but possibly different bias parameters in each step as $\beta_{n,t} \in [\beta_n, 1], \beta_{p,t} \in [\beta_{p}, 1] \text{ and } \nu_{t} \in [0, \nu]$, where $\nu < 1/2$. Let $h^{*}$ be the Bayes optimal classifier on $D$ and let $\tilde{h}_{EO,t}$ be the optimal equal opportunity classifier on $\tilde{D}_{t}$. If $\tilde{h}_{EO, t} \equiv h^{*}$, then the data distribution and bias parameters must satisfy
\begin{align*} 
\frac{(1 - 2\delta)(1 - r)}{(1 - \delta)r} - \frac{\delta}{1 - \delta} > \frac{1}{\beta_{n}^{t}} - 2\beta_{p}^{t}(1 - \nu)^{t} \quad \text{and} \\
\frac{(1- 2\delta)(1 - r)}{\delta r} > (1 - \nu)^{t}\beta_{p}^{t}(1 - \beta_{p})\frac{1 - \beta_{p}^{t}(1 - \nu)^{t}}{1 - \beta_{p}(1 - \nu)} \\
 - \frac{\beta_{n}^{t}}{\delta} - 2.
\end{align*}
\end{theorem}

The proof of Theorem \ref{thm:eo-time-varying} is given in Appendix \ref{appndx:proofs-time-varying}. The first condition in Theorem \ref{thm:eo-time-varying} can be used to get the following upper bound on the time horizon up to which $\tilde{h}_{EO,t} \equiv h^{*}$ is possible. 
\begin{corollary} \label{corr: simple_eo_time}
The conditions in Theorem \ref{thm:eo-time-varying} above are satisfied only if $t < \dfrac{\log K_{D}}{\log (1/\beta_{n})}$, where $K_{D} = \dfrac{(1 - 2\delta)(1 - r)}{(1 - \delta)r} - \dfrac{\delta}{1 - \delta} + 2$.
\end{corollary}

The above corollary can be obtained by noting that the term $2\beta_{p}^{t}(1 - \nu)^{t}$ is upper bounded by $2$ since it converges to $1$, and any $t$ greater than the value in the corollary violates the first inequality in Theorem \ref{thm:eo-time-varying}. We derive similar time-varying bias recovery conditions for demographic parity in Theorem \ref{thm:dp-time-varying} given in Appendix \ref{appndx:proofs-time-varying}.

\section{Impact Statement}
This paper presents work that aims to advance the field of Machine Learning. Our work has many potential societal consequences; pinpointing specific ones will be difficult here. Our work theoretically investigates the feasibility of using fair classification on extremely biased data as a method to recover optimal and fair classifiers on the original data. Historical, socio-cultural, implicit biases and other systematic biases in real-world data are results of complex interactions over time, and the simplistic data bias models studied in our work are insufficient to represent them truthfully. Our work underlines the need to study various possible ways to rectify data biases in algorithmic decision-making with societal consequences.

\bibliographystyle{plain}
\bibliography{mybibfile}

\newpage
\appendix
\onecolumn
\section{Proofs}

\subsection{Proof of Proposition \ref{prop:order-preserving}} \label{appndx:proof_order}

\begin{proof}
$\tilde{\eta}(x_{1}, a) \leq \tilde{\eta}(x_{2}, a)$ iff $(P \eta(x_{1}, a) + Q)(R \eta(x_{2}, a) + S) \leq (P \eta(x_{2}, a) + Q)(R \eta(x_{1}, a) + S)$, using $R\eta(x, a) + S \geq 0$, for all $0 \leq \eta(x, a) \leq 1$. By canceling the common terms, the above inequality holds iff $(PS - QR) \eta(x_{1}, a) \leq  (PS - QR) \eta(x_{2}, a)$, or equivalently $\eta(x_{1}, a) \leq \eta(x_{2}, a)$ because $PS-QR \geq 0$.
\end{proof}

\subsection{Derivations for linear fractional transforms of Data Bias models} \label{appndx: lf_example_transform}

\begin{proposition} \label{prop:blum_stangl_example}
Consider a biased distribution $\tilde{D}$ obtained from the original distribution $D$ by the following process defined by bias parameters $\beta_{p}, \beta_{n}, \nu \in (0, 1)$ \cite{blum2019recovering}. A random data point $(X, A, Y)$ from $D$ with $A=1$ remains unchanged, the points with $A=0, Y=1$ have an independent survival probability of $\beta_{p}$, and the points with $A=0, Y=0$ have an independent survival probability of $\beta_{n}$. Finally, the survived points with $A=0, Y=1$ keep their class label $1$ with probability $1-\nu$, and it gets flipped to $0$ with probability $\nu$. For the privileged group $A=1$, we have $\tilde{\eta}(x, 1) = \eta(x, 1)$, for all $x \in \mathcal{X}$, whereas for the underprivileged group $A=0$ we have $\tilde{\eta}(x, 0) = \frac{(1-\nu) \eta(x, 0)}{(1 - c)\eta(x, 0) + c}, \qquad \text{where $c = \frac{\beta_{n}}{\beta_{p}}$.}$.
    
\end{proposition}
\begin{proof}
    \begin{align*}
    \tilde{\eta}(x, 0) & = \prob{\tilde{Y}=1|X=x, A=0} \\
    & = \frac{(1-\nu)\beta_{p} \eta(x, 0)}{\beta_{p} \eta(x, 0) + \beta_{n}(1 - \eta(x, 0))} \\
    & = \frac{(1-\nu)\beta_{p} \eta(x, 0)}{(\beta_{p} - \beta_{n})\eta(x, 0) + \beta_{n}} \\
    & = \frac{(1-\nu) \eta(x, 0)}{(1 - c)\eta(x, 0) + c}, \qquad \text{where $c = \frac{\beta_{n}}{\beta_{p}}$.}
    \end{align*}
    Since no bias acts on group $A=1$, $\tilde{\eta}(x, 1) = \eta(x, 1)$.
\end{proof}

\begin{proposition} \label{prop:wang_example}
    Consider a biased distribution $\tilde{D}$ obtained form the original distribution $D$ by introducing a group dependent label flip rate \cite{dai2020label, wang2021fair}: $\epsilon_{a}^{1} = \prob{\tilde{Y} = 0 | Y = 1, A=a} \text{ and } \epsilon_{a}^{0} = \prob{\tilde{Y} = 1 | Y = 0, A=a}$, with $0 \leq \epsilon_{a}^{1} + \epsilon_{a}^{0} < 1$. Then, the observed labels in $\tilde{D}$ obey the following relationship:  $\tilde{y}_{i} = y_{i} \text{ with } 1 - \epsilon_{a_{i}}^{\id{y_{i}=1}} \text{ probability and } \tilde{y}_{i} = \neg y_{i} \text{ with probability } \epsilon_{a_{i}}^{\id{y_{i}=1}}$. In terms of a linear fractional transform:
    \begin{align*}
        \tilde{\eta}(x,a) = (1 - \epsilon^{1}_{a} - \epsilon^{0}_{a})\eta(x,a) + \epsilon^{0}_{a}
    \end{align*}
\end{proposition}

\begin{proof}
\begin{align*}
    \tilde{\eta}(x, 0) & = \prob{\tilde{Y}=1|X=x, A=0} \\
    & = \dfrac{\prob{\Tilde{Y}=1, Y=0, X=x, A=a} + \prob{\Tilde{Y}=1, Y=1, X=x, A=a}}{\prob{X=x, A=a}} \\
    & = \prob{\Tilde{Y}=1| Y=0, A=a}(1 - \eta(x,a)) + \prob{\Tilde{Y}=1|Y=1, A=a}\eta(x,a) \\
    & = (1 - \epsilon^{1}_{a} - \epsilon^{0}_{a})\eta(x,a) + \epsilon^{0}_{a}
\end{align*}
\end{proof}


\begin{proposition} \label{prop:biswas_example}
    \cite{biswas2021ensuring} Consider a biased distribution $\tilde{D}$ obtained from the original distribution $D$ by introducing a group-dependent prior probability shift such that $\tilde{A} = A$ and $\prob{\tilde{X}=x|\tilde{Y}=i, A=a}=\prob{X=x|Y=i, A=a}$, for any $i, a \in \{0, 1\}$, but $\prob{\tilde{Y}=i|A=a} \neq \prob{Y=i|A=a}$. For this data bias model, we prove that ~$\tilde{\eta}(x, a) = \dfrac{\eta(x, a)}{(1-\alpha) \eta(x, a) + \alpha}$, where $\alpha = \dfrac{\prob{\tilde{Y}=0|A=a} \prob{Y=1|A=a}}{\prob{\tilde{Y}=1|A=a} \prob{Y=0|A=a}}$.
\end{proposition}

\begin{proof}
    \begin{align*}
\tilde{\eta}(x,a) & = \prob{\tilde{Y}=1|\tilde{X}=x, A=a} \\
& = \frac{\prob{\tilde{Y}=1, \tilde{X}=x, A=a}}{\prob{\tilde{X}=x, A=a}} \\
& \text{Since $\prob{X=x|\tilde{Y}=1, A=a}=\prob{X=x|Y=1, A=a}$ from the definition of the bias model, we can write: } \\
& = \frac{\prob{X=x|\tilde{Y}=1, A=a} \prob{\tilde{Y}=1|A=a}}{\prob{X=x|\tilde{Y}=1, A=a} \prob{\tilde{Y}=1|A=a} + \prob{X=x|\tilde{Y}=0, A=a} \prob{\tilde{Y}=0|A=a}} \\
& = \frac{1}{1 + \dfrac{\prob{X=x|\tilde{Y}=0, A=a} \prob{\tilde{Y}=0|A=a}}{\prob{X=x|\tilde{Y}=1, A=a} \prob{\tilde{Y}=1|A=a}}}.
\end{align*}
Thus,
\[
\frac{1}{\tilde{\eta}(x, a)} - 1 = \frac{\prob{X=x|\tilde{Y}=0, A=a} \prob{\tilde{Y}=0|A=a}}{\prob{X=x|\tilde{Y}=1, A=a} \prob{\tilde{Y}=1|A=a}}. 
\]
Similarly,
\[
\frac{1}{\eta(x, a)} - 1 = \frac{\prob{X=x|Y=0, A=a} \prob{Y=0|A=a}}{\prob{X=x|Y=1, A=a} \prob{Y=1|A=a}}.
\]
Hence, 
\[
\frac{1}{\tilde{\eta}(x, a)} - 1 = \alpha \left(\frac{1}{\eta(x, a)} - 1\right), \quad \text{where $\alpha = \frac{\prob{\tilde{Y}=0|A=a} \prob{Y=1|A=a}}{\prob{\tilde{Y}=1|A=a} \prob{Y=0|A=a}}$}.
\]
In other words,
\[
\tilde{\eta}(x, a) = \frac{\eta(x, a)}{(1-\alpha) \eta(x, a) + \alpha}.
\]
\end{proof}

\subsection{Proofs for Section \ref{subsec:biased-TPR-TNR}} \label{appndx:tpr_tnr_proofs}

Given any classifier $f: \mathcal{X} \times \mathcal{A} \rightarrow \mathcal{Y}$, let $TPR_{a}(f)$ and $\widetilde{TPR}_{a}(f)$ denote its true positive rates conditioned on $A=a$ for distributions $D$ and $\tilde{D}$, respectively. Then
\[
TPR_{a}(f) = \frac{\prob{f(X, a)=1, A=a, Y=1}}{\prob{Y=1, A=a}} = \frac{\expecto{X|A=a}{f(X, a) \eta(X, a)}}{\expecto{X|A=a}{\eta(X, a)}},
\]
and

\begin{align*}
\widetilde{TPR}_{a}(f) & = \frac{\prob{f(X, a)=1, A=a, \tilde{Y}=1}}{\prob{\tilde{Y}=1, A=a}} = \frac{\expecto{X|A=a}{f(X, a) \tilde{\eta}(X, a)}}{\expecto{X|A=a}{\tilde{\eta}(X, a)}}.
\end{align*}

We can now prove Proposition \ref{prop:tilde-tpr-tnr}.

\begin{proof} {\bf (Proof of Proposition \ref{prop:tilde-tpr-tnr})}
\begin{align*}
& \widetilde{TPR}_{a}(h) = \prob{h(X, a)=1 | \tilde{Y}=1, A=a} \\
& = \frac{\prob{h(X, a)=1, \tilde{Y}=1 | A=a}}{\prob{\tilde{Y}=1 | A=a}} \\
& = \frac{\prob{Y=0 | A=a} \prob{h(X, a)=1, \tilde{Y}=1 | Y=0, A=a} + \prob{Y=1 | A=a} \prob{h(X, a)=1, \tilde{Y}=1 | Y=1, A=a}}{\prob{\tilde{Y}=1 | A=a}} \\
& = \frac{\prob{Y=0 | A=a} \prob{\tilde{Y}=1 | h(X, a)=1, Y=0, A=a} \prob{h(X, a)=1 | Y=0, A=a}}{\prob{\tilde{Y}=1 | A=a}} \\
& \qquad \qquad + \frac{\prob{Y=1 | A=a} \prob{\tilde{Y}=1 | h(X, a)=1, Y=1, A=a} \prob{h(X, a)=1 | Y=1, A=a}}{\prob{\tilde{Y}=1 | A=a}} \\
& = \frac{\prob{Y=0 | A=a} \prob{\tilde{Y}=1 | Y=0, A=a} \prob{h(X, a)=1 | Y=0, A=a}}{\prob{\tilde{Y}=1 | A=a}} \\
& \qquad \qquad + \frac{\prob{Y=1 | A=a} \prob{\tilde{Y}=1 | Y=1, A=a} \prob{h(X, a)=1 | Y=1, A=a}}{\prob{\tilde{Y}=1 | A=a}} \\
& \hspace{8cm} \text{because $\tilde{Y}$ depends only on $A$ and $Y$} \\
& = \frac{\prob{\tilde{Y}=1, Y=0 | A=a}}{\prob{\tilde{Y}=1 | A=a}}~ FPR_{a}(h) + \frac{\prob{\tilde{Y}=1, Y=1 | A=a}}{\prob{\tilde{Y}=1 | A=a}}~ TPR_{a}(h) \\
& = \prob{Y=0 | \tilde{Y}=1, A=a}~ FPR_{a}(h) + \prob{Y=1 | \tilde{Y}=1, A=a}~ TPR_{a}(h).
\end{align*}

Similarly, we show that $\widetilde{TNR}_{a}(h)$ can be written as a linear combination of $TNR_{a}(h)$ and $FNR_{a}(h)$.
\begin{align*}
& \widetilde{TNR}_{a}(h) \\
& = \prob{h(X, a)=a | \tilde{Y}=0, A=a} \\
& = \frac{\prob{h(X, a)=0, \tilde{Y}=0 | A=a}}{\prob{\tilde{Y}=0 | A=a}} \\
& = \frac{\prob{Y=0 | A=a} \prob{h(X, a)=0, \tilde{Y}=0 | Y=0, A=a} + \prob{Y=1 | A=a} \prob{h(X, a)=0, \tilde{Y}=0 | Y=1, A=a}}{\prob{\tilde{Y}=0 | A=a}} \\
& = \frac{\prob{Y=0 | A=a} \prob{\tilde{Y}=0 | h(X, a)=0, Y=0, A=a} \prob{h(X, a)=0 | Y=0, A=a}}{\prob{\tilde{Y}=0 | A=a}} \\
& \qquad \qquad + \frac{\prob{Y=1 | A=a} \prob{\tilde{Y}=0 | h(X, a)=0, Y=1, A=a} \prob{h(X, a)=0 | Y=1, A=a}}{\prob{\tilde{Y}=0 | A=a}} \\
& = \frac{\prob{Y=0 | A=a} \prob{\tilde{Y}=0 | Y=0, A=a} \prob{h(X, a)=0 | Y=0, A=a}}{\prob{\tilde{Y}=0 | A=a}} \\
& \qquad \qquad + \frac{\prob{Y=1 | A=a} \prob{\tilde{Y}=0 | Y=1, A=a} \prob{h(X, a)=0 | Y=1, A=a}}{\prob{\tilde{Y}=0 | A=a}} \\
& \hspace{8cm} \text{because $\tilde{Y}$ depends only on $A$ and $Y$} \\
& = \frac{\prob{\tilde{Y}=0, Y=0 | A=a}}{\prob{\tilde{Y}=0 | A=a}}~ TNR_{a}(h) + \frac{\prob{\tilde{Y}=0, Y=1 | A=a}}{\prob{\tilde{Y}=0 | A=a}}~ FNR_{a}(h) \\
& = \prob{Y=0 | \tilde{Y}=0, A=a}~ TNR_{a}(h) + \prob{Y=1 | \tilde{Y}=0, A=a}~ FNR_{a}(h).
\end{align*}
\end{proof}


\begin{proof} {\bf (Proof of Proposition \ref{prop:biased-eo-threshold})}
We begin with a known technique that uses Lagrange duality to give threshold-based characterization of optimal equal opportunity classifiers \cite{menon2018cost, chzhen2019leveraging}. We define $\tilde{h}_{EO} = \underset{h \in \widetilde{\mathcal{H}}_{\text{fair, EO}}}{\mathrm{argmax}}~ \prob{h(X, A)=\tilde{Y}}$. By Lagrange duality, we can write
\begin{align*}
\tilde{h}_{EO} & = \underset{h \in \mathcal{H}}{\mathrm{argmax}}~ \underset{\lambda \in \R}{\min}~ \prob{h(X, A)=\tilde{Y}} + \lambda~ \prob{h(X, A)=1 | \tilde{Y}=1, A=0} - \lambda~ \prob{h(X, A)=1 | \tilde{Y}=1, A=1} \\
& = \underset{h \in \mathcal{H}}{\mathrm{argmax}}~ \underset{\lambda \in \R}{\min}~ \sum_{a=0}^{1} \prob{A=a}~ \expecto{X|A=a}{h(X, a) (2\tilde{\eta}(X, a) - 1)} + \frac{(-1)^{\id{a=1}}~ \lambda}{\expecto{X|A=a}{\tilde{\eta}{(X, a)}}}~ \expecto{X|A=a}{h(X, a)\tilde{\eta}(X, a)} \\
& = \underset{h \in \mathcal{H}}{\mathrm{argmax}}~ \underset{\lambda \in \R}{\min}~ \sum_{a=0}^{1} \prob{A=a}~ \expecto{X|A=a}{h(X, a) \left(\left(2 + \dfrac{(-1)^{\id{a=1}} \lambda}{\prob{\tilde{Y}=1, A=a}}\right) \tilde{\eta}(X, a) - 1\right)} \\
& = \underset{\lambda \in \R}{\min}~ \underset{h \in \mathcal{H}}{\mathrm{argmax}}~  \sum_{a=0}^{1} \prob{A=a}~ \expecto{X|A=a}{h(X, a) \left(\left(2 + \dfrac{(-1)^{\id{a=1}} \lambda}{\prob{\tilde{Y}=1, A=a}}\right) \tilde{\eta}(X, a) - 1\right)},
\end{align*}
by Sion's minimax theorem as the objective is linear in $\lambda$ and $h$. Thus, there exists $\lambda^{*} \in \R$ such that the objective on $h$ splits group-wise with the optimal solution given by
\begin{align*}
\tilde{h}_{EO}(x, a) & = \id{\tilde{\eta}(x, a) \geq \dfrac{1}{2 + \dfrac{(-1)^{\id{a=1}} \lambda^{*}}{\prob{\tilde{Y}=1, A=a}}}} \\
& = \begin{cases} \id{\dfrac{(1-\nu)\eta(x, 0)}{(1-c)\eta(x, 0) + c} \geq \dfrac{1}{2 + \dfrac{\lambda^{*}}{\prob{\tilde{Y}=1, A=0}}}}, \quad \text{for $a=0$} \qquad \text{(Example \ref{eg:blumStangl})} \\
\id{\eta(x, 1) \geq \dfrac{1}{2 - \dfrac{\lambda^{*}}{\prob{Y=1, A=1}}}}, \quad \text{for $a=1$}
\end{cases} \\
& = \begin{cases} \id{\eta(x, 0) \geq \dfrac{c}{1- 2\nu + c + \dfrac{\lambda^{*} (1-\nu)}{\prob{\tilde{Y}=1, A=0}}}}, \quad \text{for $a=0$} \\
\id{\eta(x, 1) \geq \dfrac{1}{2 - \dfrac{\lambda^{*}}{\prob{Y=1, A=1}}}}, \quad \text{for $a=1$}
\end{cases} \\
\end{align*}
Let $\prob{A=0} = r$ and $\prob{A=1} = 1-r$, and let $\prob{Y=1|A=0} = \prob{Y=1|A=1} = q$. Then $\prob{Y=1, A=1} = \prob{A=1} \prob{Y=1|A=1} = (1-r)q$ and 
\begin{align*}
\prob{\tilde{Y}=1, A=0} & = \prob{A=0} \prob{Y=1|A=0} \prob{\tilde{Y}=1 | Y=1, A=0} \\
& \qquad \qquad + \prob{A=0} \prob{Y=0|A=0} \prob{\tilde{Y}=1 | Y=0, A=0} \\
& = rq \beta_{p} (1-\nu).
\end{align*}
Therefore,
\begin{align*}
\tilde{h}_{EO}(x, a) & = \begin{cases} \id{\eta(x, 0) \geq \dfrac{c}{1- 2\nu + c + \dfrac{\lambda^{*}}{\beta_{p} rq}}}, \quad \text{for $a=0$} \\
\id{\eta(x, 1) \geq \dfrac{1}{2 - \dfrac{\lambda^{*}}{(1-r)q}}}, \quad \text{for $a=1$}
\end{cases}.    
\end{align*}
Equivalently, we can also write
\begin{align*}
\tilde{h}_{EO}(x, a) & = \begin{cases} \id{\eta(x, 0) \geq \dfrac{1}{1 + \dfrac{1- 2\nu}{c} + \dfrac{\lambda^{*}}{\beta_{n} rq}}}, \quad \text{for $a=0$} \\
\id{\eta(x, 1) \geq \dfrac{1}{2 - \dfrac{\lambda^{*}}{(1-r)q}}}, \quad \text{for $a=1$}.
\end{cases}.    
\end{align*}
\end{proof}

We can derive similar results for demographic parity \cite{menon2018cost,chzhen2019leveraging}.

\begin{proposition} \label{prop:biased-dp-threshold}
For any distribution $D$ and its biased version $\tilde{D}$ described above, let $\tilde{h}_{DP}$ be a classifier of the maximum accuracy among all binary classifiers that satisfy demographic parity on $\tilde{D}$. Then there exists $\lambda^{*} \in \R$ such that
\[
\tilde{h}_{DP}(x, a) = \begin{cases} \id{\eta(x, 0) \geq \dfrac{c(r - \lambda^{*})}{2r(1 - \nu) - (1 - c)(r - \lambda^{*})}}, \quad \text{for $a=0$} \\
\id{\eta(x, 1) \geq \dfrac{1}{2} + \dfrac{\lambda^{*}}{2(1 - r)}}, \quad \text{for $a=1$}
\end{cases}    
\]
\end{proposition}

\begin{proof}
Similar to previous work that gives a threshold-based characterization of optimal demographic parity classifiers using Lagrange duality \cite{menon2018cost}, we use the same idea on the biased distribution.
$\tilde{h}_{DP} = \underset{h \in \widetilde{\mathcal{H}}_{\text{fair, DP}}}{\mathrm{argmax}}~ \prob{h(X, A)=\tilde{Y}}$. By Lagrange duality, we can write
\begin{align*}
\tilde{h}_{DP} & = \underset{h \in \mathcal{H}}{\mathrm{argmax}}~ \underset{\lambda \in \R}{\min}~ \prob{h(X, A)=\tilde{Y}} + \lambda~ \prob{h(X, A)=1 | A=0} - \lambda~ \prob{h(X, A)=1 | A=1} \\
& = \underset{h \in \mathcal{H}}{\mathrm{argmax}}~ \underset{\lambda \in \R}{\min}~ \sum_{a=0}^{1} \prob{A=a}~ \expecto{X|A=a}{h(X, a) (2\tilde{\eta}(X, a) - 1)} + (-1)^{\id{a=1}}~ \lambda~ \expecto{X|A=a}{h(X, a)} \\
& = \underset{h \in \mathcal{H}}{\mathrm{argmax}}~ \underset{\lambda \in \R}{\min}~ \sum_{a=0}^{1} \prob{A=a}~ \expecto{X|A=a}{h(X, a) \left(2 \tilde{\eta}(X, a) - 1 + \dfrac{(-1)^{\id{a=1}}~ \lambda}{\prob{A=a}}\right)} \\
& = \underset{\lambda \in \R}{\min}~ \underset{h \in \mathcal{H}}{\mathrm{argmax}}~  \sum_{a=0}^{1} \prob{A=a}~ \expecto{X|A=a}{h(X, a) \left(2 \tilde{\eta}(X, a) - 1 + \dfrac{(-1)^{\id{a=1}}~ \lambda}{\prob{A=a}}\right)},
\end{align*}
by Sion's minimax theorem as the objective is linear in $\lambda$ and $h$. 
Thus, there exists some optimal $\lambda^{*} \in \R$ such that the objective on $h$ splits group-wise with the optimal solution given by
\begin{align*}
\tilde{h}_{DP}(x, a) & = \id{\tilde{\eta}(x, a) \geq \frac{1}{2} - \frac{(-1)^{\id{a=1}} \lambda^{*}}{2\prob{A=a}}} \\
& = \begin{cases} \id{\dfrac{(1-\nu)\eta(x, 0)}{(1-c)\eta(x, 0) + c} \geq \dfrac{1}{2} - \dfrac{\lambda^{*}}{2\prob{A=0}}}, \quad \text{for $a=0$} \quad \text{(Example \ref{eg:blumStangl})} \\
\id{\eta(x, 1) \geq \dfrac{1}{2} + \dfrac{\lambda^{*}}{2\prob{A=1}}}, \quad \text{for $a=1$}
\end{cases} \\
& = \begin{cases} \id{\eta(x, 0) \geq \dfrac{c(r - \lambda^{*})}{2r(1 - \nu) - (1 - c)(r - \lambda^{*})}}, \quad \text{for $a=0$} \\
\id{\eta(x, 1) \geq \dfrac{1}{2} + \dfrac{\lambda^{*}}{2(1 -r)}}, \quad \text{for $a=1$}
\end{cases} \\
\end{align*}
\end{proof}

\subsection{Proofs for Section \ref{subsec:massart}} \label{appndx:blumStanglProofs}

For the distribution described in Section \ref{subsec:massart}, the following result holds:

\begin{lemma} \label{thm:eta-lemma}
Let $h: \mathcal{X} \times \mathcal{A} \rightarrow \{0, 1\}$ be a deterministic classifier. Let $Y|X=x, A=a$ be a random variable that takes value $h(x, a)$ with probability $1-\delta(x, a)$ and $\neg h(x, a)$ with probability $\delta(x, a)$, for some $\delta(x, a) \leq \delta < 1/2$. Let $\eta(x, a) = \prob{Y=1|X=x, A=a}$ and let $g(x, a)$ be a threshold classifier given by $g(x, a) = \id{\eta(x, a) \geq t_{a}}$, using group-dependent thresholds $t_{0}$ and $t_{1}$, respectively. If $t_{0}, t_{1} \in (\delta, 1-\delta)$ then $g \equiv h \equiv h^{*}$, where $h^{*}$ is the Bayes optimal classifier for the joint distribution on $\mathcal{X} \times \mathcal{A} \times \mathcal{Y}$.
\end{lemma}

\begin{proof}
It is folklore that the (group-aware) Bayes optimal classifier is given by $h^{*}(x, a) = \id{\eta(x, a) \geq 1/2}$ \cite{zeng2022fair}. One of the ways by which we can recover the Bayes Optimal classifier with the biased distribution $\tilde{D}$ is when the Bayes optimal classifier on the original distribution $D$ does not change around the vicinity of $1/2$. Consider the following two cases for $h(x, a)$. If $h(x, a)=0$, then $Y|X=x, A=a$ takes value $0$ with probability $\delta(x, a) \leq \delta$, and hence, $\id{\eta(x, a) \geq t_{a}} = \id{\delta(x, a) \geq t_{a}} = \id{\delta(x, a) \geq 1/2} = 0$, for $\delta(x, a) \leq \delta < 1/2$ and $t_{a} \in (\delta, 1-\delta)$. On the other hand, if $h(x, a)=1$, then $Y|X=x, A=a$ takes value $1$ with probability $1 - \delta(x, a) \geq 1-\delta$, and hence, $\id{\eta(x, a) \geq t_{a}} = \id{1 - \delta(x, a) \geq t_{a}} = \id{1 - \delta(x, a) \geq 1/2} = 1$, for $\delta(x, a) \leq \delta < 1/2$ and $t_{a} \in (\delta, 1-\delta)$. Therefore, $g \equiv h \equiv h^{*}$.
\end{proof}

We now describe the proof of Theorem \ref{thm:dp-recovery}. The proof for Theorem \ref{thm:blumStangl-reprove} is clubbed with the proof of Theorem \ref{thm:blum_stangl_eo_massart} later in this section.

\begin{proof} {\bf (Proof of Theorem \ref{thm:dp-recovery})}
We use the characterization optimal demographic parity classifier obtained in Proposition \ref{prop:biased-dp-threshold} for $\tilde{D}$, i.e., there exists $\lambda^{*} \in \R$ such that
\[
\tilde{h}_{DP}(x, a) = \begin{cases} \id{\eta(x, 0) \geq \dfrac{c(r - \lambda^{*})}{2r(1 - \nu) - (1 - c)(r - \lambda^{*})}}, \quad \text{for $a=0$} \\
\id{\eta(x, 1) \geq \dfrac{1}{2} + \dfrac{\lambda^{*}}{2(1 - r)}}, \quad \text{for $a=1$}.
\end{cases}    
\]

\noindent {\bf Step 1: If the given conditions on the bias parameters hold, then there exists a $\lambda \in \R$ such that $h_{\lambda} = h^{*}$.} 
From Lemma \ref{thm:eta-lemma}, we know that the threshold on $\eta(x, 1)$ lies in the interval $(\delta, 1-\delta)$ if and only if $\delta < \dfrac{1}{2} + \dfrac{\lambda^{*}}{2(1 - r)}  < 1 - \delta$, or equivalently, $(2\delta - 1)(1 - r) < \lambda^{*} < (1 - 2\delta)(1 - r)$. Similarly, the threshold on $\eta(x, 0)$ lies in the interval $(\delta, 1- \delta)$ if and only if $\delta < \dfrac{c(r - \lambda^{*})}{2r(1 - \nu) - (1 - c)(r - \lambda^{*})} < 1 - \delta$, or equivalently,
\[
r\left(1 - \dfrac{2(1- \delta)(1 - \nu)}{c + (1 - c)(1 - \delta)}\right) < \lambda^{*} < r\left(1 - \dfrac{2\delta(1 - \nu)}{\delta(1 - c) + c}\right).
\]
There exists $\lambda^{*}$ that simultaneously satisfies both sets of constraints given above if and only if
\begin{align*}
r\left(1 - \dfrac{2(1- \delta)(1 - \nu)}{c + (1 - c)(1 - \delta)}\right) < (1 - 2\delta)(1 - r) \text{ and }
(2\delta - 1)(1 - r) < r\left(1 - \dfrac{2\delta(1 - \nu)}{\delta(1 - c) + c} \right).
\end{align*}
Using $c=\beta_{n}/\beta_{p}$, the above conditions can be simplified as
\begin{align*}
\beta_{p}(1 - \delta)(1 - 2\delta - 2r(\nu - \delta)) + \delta\beta_{n}(1 - 2\delta -2r(1 - \delta)) > 0  \text{ and }\\
\beta_{p}\delta(1 - 2r(1 - \nu) - 2\delta(1 - r)) + (1 - \delta)\beta_{n}(1 - 2\delta(1 - r)) > 0
\end{align*}

\noindent {\bf Step 2: If the given conditions on the bias parameters hold, then there cannot exist any $\lambda \in \R$ such that both the group-wise thresholds applied to $\eta(x, a)$ in $h_{\lambda}$ are at most $\delta$, or both the thresholds are at least $1-\delta$. (Thresholds lying on the same side)}
The threshold on $\eta(x, 1)$ is at most $\delta$ if and only if $\dfrac{1}{2} + \dfrac{\lambda}{2(1 - r)} \leq \delta$, or equivalently, $\lambda \leq (2\delta - 1)(1 - r)$. Similarly, the threshold on $\eta(x, 0)$ is at most  $\delta$ if and only if $\dfrac{c(r - \lambda)}{2r(1 - \nu) - (1 - c)(r - \lambda)} \leq \delta$, or equivalently, $\lambda \geq r\left( 1 - \dfrac{2\delta\beta_{p}(1 - \nu)}{\beta_{n}(1 - \delta) + \delta\beta_{p}} \right)$. If both were to hold simultaneously, then
$r\left( 1 - \dfrac{2\delta\beta_{p}(1 - \nu)}{\beta_{n}(1 - \delta) + \delta\beta_{p}} \right) \leq (2\delta - 1)(1 - r)$, or equivalently, $\beta_{p}\delta(1 - 2r(1 - \nu) - 2\delta(1 - r)) + (1 - \delta)\beta_{n}(1 - 2\delta(1 - r)) \leq 0$, which would violate the second condition in our theorem. 

On the other hand, the threshold on $\eta(x, 1)$ is at least interval $1-\delta$ if and only if $\dfrac{1}{2} + \dfrac{\lambda}{2(1 - r)} \geq 1 - \delta$, or equivalently, $\lambda \geq (1  - 2\delta)(1 - r)$. Similarly, the threshold on $\eta(x, 0)$ is at least $\dfrac{c(r - \lambda)}{2r(1 - \nu) - (1 - c)(r - \lambda)} \geq 1 - \delta$, or equivalently, $\lambda \leq r\left( 1 - \dfrac{2(1 - \delta)(1 - \nu)}{c + (1 - c)(1 - \delta)} \right)$. If both were to hold simultaneously, then $(1  - 2\delta)(1 - r) \leq r\left( 1 - \dfrac{2(1 - \delta)(1 - \nu)}{c + (1 - c)(1 - \delta)} \right)$, or equivalently, $\beta_{p}(1 - \delta)(1 - 2\delta - 2r(\nu - \delta)) + \delta\beta_{n}(1 - 2\delta -2r(1 - \delta)) \leq 0$, which would violate the first condition in our theorem. \\

 \noindent {\bf Step 3: For any $\lambda$, if the group-wise thresholds on $\eta(x, a)$ in $h_{\lambda}$ have one of them at most $\delta$ and another at least $1-\delta$, then $h_{\lambda}$ cannot satisfy demographic parity unless $h_{\lambda} = h^{*}$. (Thresholds lying on opposite sides)} We know that $h^{*}$ satisfies demographic parity. Let $PR_{a}(h_{\lambda}) = \prob{h_{\lambda} = 1 | A = a}$ (group positive rate). For demographic parity, we require $PR_{0} = PR_{1}$. Suppose the threshold of $h_{\lambda}$ on $\eta(x,0) (\text{call it } t_{0})$ and the threshold on $\eta(x,1) (\text{call it } t_{1})$ are separated by either $\delta$ or $1 - \delta$. WLOG assume $t_{0} < t_{1}$. Suppose $t_{0} \leq \delta \leq t_{1}$. Since there is no input $(x, a)$ with $\eta(x, a) \in (\delta, 1-\delta)$, we get $PR_{0}(\tilde{h}_{\lambda}) \geq PR_{0}(h^{*})$ and $PR_{1}(h^{*}) \geq PR_{1}(\tilde{h}_{\lambda})$. Since $h^{*}$ satisfies demographic parity on the original distribution $D$, we have $PR_{0}(h^{*}) = PR_{1}(h^{*})$. If $h_{\lambda}$ satisfies demographic parity on the biased distribution $\tilde{D}$, we have $\widetilde{PR}_{0}(h_{\lambda}) = \widetilde{PR}_{1}(h_{\lambda})$, and since $\widetilde{PR}_{a}(h_{\lambda}) = PR_{a}(h_{\lambda})$ by Corollary 3.6, it implies $PR_{0}(h_{\lambda}) = PR_{1}(h_{\lambda})$. 
Combining the two observations above, we get $PR_{0}(h_{\lambda}) = PR_{0}(h^{*}) = PR_{1}(h^{*}) = PR_{1}(h_{\lambda})$, which is possible only if $h_{\lambda} \equiv h^{*}$ (as both are group-wise threshold classifiers). 
The other case $t_{0} \leq 1-\delta \leq t_{1}$ can be argued similarly. \\

\end{proof}


Now, we give a complete proof of Theorem \ref{thm:blum_stangl_eo_massart}. Note that Theorem \ref{thm:blum_stangl_eo_massart} implies Theorem \ref{thm:blumStangl-reprove} as i.i.d. noise is a special case of Massart noise when we have $\delta(x, a) = \delta$, for all $(x, a) \in \mathcal{X} \times \mathcal{A}$. Thus, Theorem \ref{thm:blum_stangl_eo_massart} is a stronger statement than Theorem \ref{thm:blumStangl-reprove} (the original recovery theorem of Blum \& Stangl \cite{blum2019recovering}) while having the same necessary and sufficient conditions on the data and bias parameters. 

\begin{proof}{\bf (Proof of Theorem \ref{thm:blum_stangl_eo_massart})} Let
\begin{align*}
\tilde{h}_{\lambda}(x, a) & = \begin{cases} \id{\eta(x, 0) \geq \dfrac{1}{1 + \dfrac{1- 2\nu}{c} + \dfrac{\lambda}{\beta_{n} rq}}}, \quad \text{for $a=0$} \\
\id{\eta(x, 1) \geq \dfrac{1}{2 - \dfrac{\lambda}{(1-r)q}}}, \quad \text{for $a=1$}.
\end{cases}    
\end{align*}
From the characterization of optimal equal opportunity classifier shown earlier, there exists some optimal $\lambda^{*} \in \R$ such that $\tilde{h}_{EO} = h_{\lambda^{*}}$. We will show that if the conditions in Theorem 4.1 of Blum-Stangl hold, then the only possible choices left for $\lambda^{*} \in \R$ are the ones that give $h_{\lambda^{*}} = h^{*}$.

\noindent {\bf Step 1: If the given conditions on the bias parameters hold, then there exists a $\lambda \in \R$ such that $\tilde{h}_{\lambda} = h^{*}$.} The threshold on $\eta(x, 1)$ lies in the interval $(\delta, 1-\delta)$ if and only if $\dfrac{1}{1-\delta} < 2 - \dfrac{\lambda}{(1-r)q}  < \dfrac{1}{\delta}$, or equivalently, $\dfrac{-(1 - 2\delta)(1-r) q}{\delta} < \lambda < \dfrac{(1-2\delta)(1-r) q}{1-\delta}$. Similarly, the threshold on $\eta(x, 0)$ lies in the interval $(\delta, 1- \delta)$ if and only if $\dfrac{1}{1-\delta} < 1 + \dfrac{1-2\nu}{c} + \dfrac{\lambda}{\beta_{n} r q} < \dfrac{1}{\delta}$, or equivalently,
\[
\beta_{n}rq \left(\frac{\delta}{1-\delta} - \frac{1-2\nu}{c}\right) < \lambda < \beta_{n} rq \left(\frac{1-\delta}{\delta} - \frac{1-2\nu}{c}\right).
\]
There exists $\lambda$ that simultaneously satisfies both sets of constraints given above if and only if
\begin{align*}
\beta_{n}rq \left(\frac{\delta}{1-\delta} - \frac{1-2\nu}{c}\right) & < \dfrac{(1-2\delta)(1-r)q}{1-\delta} \\
\dfrac{-(1 - 2\delta)(1-r)q}{\delta} & < \beta_{n} rq \left(\frac{1-\delta}{\delta} - \frac{1-2\nu}{c}\right).
\end{align*}
Using $c=\beta_{n}/\beta_{p}$, the above conditions can be rewritten as
\begin{align*}
(1-r)(1-2\delta) + r\left((1-\delta) \beta_{p} (1-2\nu) - \delta \beta_{n}\right) & > 0 \quad \text{and} \\
(1-r)(1-2\delta) + r\left((1-\delta) \beta_{n} - \delta \beta_{p} (1-2\nu)\right) & > 0.
\end{align*}

\noindent {\bf Step 2: If the given conditions on the bias parameters hold, then there cannot exist any $\lambda \in \R$ such that both the group-wise thresholds applied to $\eta(x, a)$ in $h_{\lambda}$ are at most $\delta$, or both the thresholds are at least $1-\delta$. (Thresholds lying on the same side)}
The threshold on $\eta(x, 1)$ is at most $\delta$ if and only if $2 - \dfrac{\lambda}{(1-r)q}  \geq \dfrac{1}{\delta}$, or equivalently, $\dfrac{-(1 - 2\delta)(1-r) q}{\delta} \geq \lambda$. Similarly, the threshold on $\eta(x, 0)$ is at most  $\delta$ if and only if $1 + \dfrac{1-2\nu}{c} + \dfrac{\lambda}{\beta_{n} r q} \geq \dfrac{1}{\delta}$, or equivalently, $\lambda \geq \beta_{n} rq \left(\dfrac{1-\delta}{\delta} - \dfrac{1-2\nu}{c}\right)$. If both were to hold simultaneously, then
$\dfrac{-(1 - 2\delta)(1-r) q}{\delta} \geq \beta_{n} rq \left(\dfrac{1-\delta}{\delta} - \dfrac{1-2\nu}{c}\right)$, or equivalently, $(1-r)(1-2\delta) + r\left((1-\delta) \beta_{n} - \delta \beta_{p} (1-2\nu)\right) \leq 0$, which would violate the second condition in Theorem 4.1 of Blum-Stangl. 

On the other hand, the threshold on $\eta(x, 1)$ is at least interval $1-\delta$ if and only if $\dfrac{1}{1-\delta} \geq 2 - \dfrac{\lambda}{(1-r)q}$, or equivalently, $\lambda \geq \dfrac{(1-2\delta)(1-r) q}{1-\delta}$. Similarly, the threshold on $\eta(x, 0)$ is at least $1- \delta$ if and only if $\dfrac{1}{1-\delta} \geq 1 + \dfrac{1-2\nu}{c} + \dfrac{\lambda}{\beta_{n} r q}$, or equivalently, $\beta_{n}rq \left(\dfrac{\delta}{1-\delta} - \dfrac{1-2\nu}{c}\right) \geq \lambda$. If both were to hold simultaneously, then $\beta_{n}rq \left(\dfrac{\delta}{1-\delta} - \dfrac{1-2\nu}{c}\right) \geq \dfrac{(1-2\delta)(1-r) q}{1-\delta}$, or equivalently, $(1-r)(1-2\delta) + r\left((1-\delta) \beta_{p} (1-2\nu) - \delta \beta_{n}\right) \leq 0$, which would violate the first condition in Theorem 4.1 of Blum-Stangl. \\


 \noindent {\bf Step 3: For any $\lambda$, if the group-wise thresholds on $\eta(x, a)$ in $\tilde{h}_{\lambda}$ have one of them at most $\delta$ and another at least $1-\delta$, then $\tilde{h}_{\lambda}$ cannot satisfy equal opportunity unless $\tilde{h}_{\lambda} = h^{*}$. (Thresholds lying on opposite sides)} We know that $h^{*}$ satisfies equal opportunity. Suppose the threshold of $h_{\lambda}$ on $\eta(x,0) (\text{call it } t_{0})$ and the threshold on $\eta(x,1) (\text{call it } t_{1})$ are separated by either $\delta$ or $1 - \delta$. WLOG assume $t_{0} < t_{1}$. Suppose $t_{0} \leq \delta \leq t_{1}$. Since there is no input $(x, a)$ with $\eta(x, a) \in (\delta, 1-\delta)$, we get $TPR_{0}(\tilde{h}_{\lambda}) \geq TPR_{0}(h^{*})$ and $TPR_{1}(h^{*}) \geq TPR_{1}(\tilde{h}_{\lambda})$. 
Since $h^{*}$ satisfies equal opportunity on the original distribution $D$, we have $TPR_{0}(h^{*}) = TPR_{1}(h^{*})$.
If $\tilde{h}_{\lambda}$ satisfies equal opportunity on the biased distribution $\tilde{D}$, we have $\widetilde{TPR}_{0}(\tilde{h}_{\lambda}) = \widetilde{TPR}_{1}(\tilde{h}_{\lambda})$, and since $\widetilde{TPR}_{a}(\tilde{h}_{\lambda}) = TPR_{a}(\tilde{h}_{\lambda})$ by Corollary 3.6, it implies $TPR_{0}(\tilde{h}_{\lambda}) = TPR_{1}(\tilde{h}_{\lambda})$. 
Combining the two observations above, we get $TPR_{0}(\tilde{h}_{\lambda}) = TPR_{0}(h^{*}) = TPR_{1}(h^{*}) = TPR_{1}(\tilde{h}_{\lambda})$, which is possible only if $\tilde{h}_{\lambda} \equiv h^{*}$ (as both are group-wise threshold classifiers). 
The other case $t_{0} \leq 1-\delta \leq t_{1}$ can be argued similarly. \\

\noindent {\bf Combing Steps 1, 2, and 3 to complete the proof:} Finally, from Steps 1, 2, and 3 together, it is clear that if the conditions in Theorem 4.1 of Blum-Stangl are met then the only possible classifiers $\tilde{h}_{\lambda}$ are the ones for which $\tilde{h}_{\lambda} = h^{*}$, and they satisfy equal opportunity on $\tilde{D}$. So under these conditions, the optimal $\lambda^{*}$ (whatever it may be) gives $\tilde{h}_{EO} = \tilde{h}_{\lambda^{*}} = h^{*}$. \\
\end{proof}

\subsection{Proofs for Section \ref{sec:reject_option}} 
\label{appndx:rejectionProofs}

\begin{proof} {\bf(Proof of Proposition \ref{prop:opt-rej})}

The optimal reject option classifier $g: \mathcal{X} \times \mathcal{A} \rightarrow \{0, 1, \bot\}$ with rejection penalty $\delta$ can be thought of as a pair of binary classifiers $\rho$ and $h$ in $\mathcal{H}$ such that:

\[
\rho(x, a) = \begin{cases} 1, \quad \text{if $g(x, a) = \bot$} \\ 0, \quad \text{if $g(x, a) \neq \bot$}
\end{cases}
\]
and $g(x, a) = h(x, a)$ whenever $\rho(x, a) = 0$ (or equivalently, $g(x, a) \neq \bot$). The optimal reject option classifier is defined as
\[
h^{\text{rej}} = \underset{g \in \mathcal{H}^{\text{rej}}}{\mathrm{argmin}}~ \prob{g(X, A) \neq Y, g(X, A) \neq \bot} + \delta~ \prob{g(X, A) = \bot},
\]
where $\mathcal{H}^{\text{rej}}$ be the hypothesis class of all reject option classifiers $g: \mathcal{X} \times \mathcal{A} \rightarrow \{0, 1, \bot\}$. Equivalently, it can be re-written as:

\begin{align*}
(\rho^{*}, h^{*}) & = \underset{\rho \in \mathcal{H}}{\mathrm{argmin}}~ \underset{h \in \mathcal{H}}{\mathrm{argmin}}~ \prob{h(X, A) \neq Y, \rho(X, A)=0} + \delta~ \prob{\rho(X, A)=1} \\ 
& = \underset{\rho \in \mathcal{H}}{\mathrm{argmin}}~ \underset{h \in \mathcal{H}}{\mathrm{argmin}}~ 1 - \prob{\rho(X,A)=1} - \prob{h(X,A)=Y, \rho(X,A)=0} + \delta~ \prob{\rho(X,A)=1} \\ 
& = \underset{\rho \in \mathcal{H}}{\mathrm{argmax}}~ \underset{h \in \mathcal{H}}{\mathrm{argmax}}~ \prob{h(X,A)=Y, \rho(X,A)=0} + (1-\delta)~ \prob{\rho(X,A)=1} \\
& = \underset{\rho \in \mathcal{H}}{\mathrm{argmax}}~ \underset{h \in \mathcal{H}}{\mathrm{argmax}}~ \sum_{a=0}^{1} \prob{A=a}~ \expecto{X|A=a}{(1-\rho(X,a))\left(h(X,a)\eta(X,a) + (1-h(X,a))(1-\eta(X,a))\right) + (1-\delta)\rho(X,a)} \\
& = \underset{\rho \in \mathcal{H}}{\mathrm{argmax}}~
\sum_{a=0}^{1} \prob{A=a}~ \expecto{X|A=a}{(1-\delta)~ \rho(X,a) + (1-\rho(X,a)) \left(1-\eta(X,a) + \underset{h \in \mathcal{H}}{\mathrm{argmax}}~ h(X,a)(2\eta(X,a)-1)\right)}. 
\end{align*}
We can now focus on obtaining the optimal functions for each group. For any $\rho \in \mathcal{H}$, solving the inner optimization gives the optimal solution $h^{\text{rej}}(x,a) = \id{\eta(x,a) \geq \frac{1}{2}}$ whenever $\rho(x,a) \neq 1$. Thus, the outer optimization can be written as
\[
\rho^{*} = \underset{\rho \in \mathcal{H}}{\mathrm{argmax}}~
\expecto{X|A=a}{(1-\delta)~ \rho(X,a) + (1-\rho(X,a)) \left(1-\eta(X,a) + \id{\eta(X,a) \geq \frac{1}{2}}(2\eta(X,a)-1)\right)}.
\]
The above maximization can be solved point-wise for each $x \in \mathcal{X}$ in group $A=a$. Whenever $\eta(x,a) \geq 1/2$ for an input $x$ from group $a$, the function inside the expectation becomes $(1-\delta)\rho(x,a) + (1-\rho(x,a))\eta(x,a)$, which is maximized by $\rho^{*}(x,a) = \id{\eta(x,a) \in [1/2, 1-\delta)}$. On the other hand, whenever $\eta(x,a) < 1/2$ for an input $x,a$, the function inside the expectation becomes $(1-\delta)\rho(x,a) + (1-\rho(x,a)(1-\eta(x,a))$, which is maximized by $\rho^{*}(x,a) = \id{\eta(x,a) \in (\delta, 1/2)}$. Thus, overall we can write $\rho^{*}(x,a) = \id{\eta(x,a) \in (\delta, 1-\delta)}$.

\end{proof}

\begin{proof} 
{\bf (Proof of Theorem \ref{thm:reject_recovery})} From Proposition \ref{prop:biased-eo-threshold}, we can write:
\begin{align*}
\tilde{h}_{\lambda}(x, a) & = \begin{cases} \id{\eta(x, 0) \geq \dfrac{1}{1 + \dfrac{1- 2\nu}{c} + \dfrac{\lambda}{\beta_{n} rq}}}, \quad \text{for $a=0$} \\
\id{\eta(x, 1) \geq \dfrac{1}{2 - \dfrac{\lambda}{(1-r)q}}}, \quad \text{for $a=1$}.
\end{cases}    
\end{align*}

We will now attempt to recover a classifier that matches the predictions of $h^{\text{rej}}$, whenever $h^{\text{rej}}(x, a) \neq \bot$, but incurs some penalty while trying to label the points when $h^{\text{rej}}(x, a) = \bot$.
From the characterization of the optimal equal opportunity classifier shown earlier, there exists some optimal $\lambda^{*} \in \R$ such that $\tilde{h}_{EO}=h^{\text{rej}}$, whenever $h^{\text{rej}}(x, a) \neq \bot$. Such a classifier will not change whenever it lies in the rejection interval $(\delta, 1 - \delta)$.

\noindent {\bf Step 1: If the given conditions on the bias parameters hold, then there exists a $\lambda \in \R$ such that $\tilde{h}_{\lambda} = h^{\text{rej}}$, whenever $h^{\text{rej}}(x, a) \neq \bot$.} The threshold on $\eta(x, 1)$ lies in the interval $(\delta, 1-\delta)$ if and only if $\dfrac{1}{1-\delta} < 2 - \dfrac{\lambda}{(1-r)q}  < \dfrac{1}{\delta}$, or equivalently, $\dfrac{-(1 - 2\delta)(1-r) q}{\delta} < \lambda < \dfrac{(1-2\delta)(1-r) q}{1-\delta}$. Similarly, the threshold on $\eta(x, 0)$ lies in the interval $(\delta, 1- \delta)$ if and only if $\dfrac{1}{1-\delta} < 1 + \dfrac{1-2\nu}{c} + \dfrac{\lambda}{\beta_{n} r q} < \dfrac{1}{\delta}$, or equivalently,

\[
\beta_{n}rq \left(\frac{\delta}{1-\delta} - \frac{1-2\nu}{c}\right) < \lambda < \beta_{n} rq \left(\frac{1-\delta}{\delta} - \frac{1-2\nu}{c}\right).
\]
There exists $\lambda$ that simultaneously satisfies both sets of constraints given above if and only if
\begin{align*}
\beta_{n}rq \left(\frac{\delta}{1-\delta} - \frac{1-2\nu}{c}\right) & < \dfrac{(1-2\delta)(1-r)q}{1-\delta} \\
\dfrac{-(1 - 2\delta)(1-r)q}{\delta} & < \beta_{n} rq \left(\frac{1-\delta}{\delta} - \frac{1-2\nu}{c}\right).
\end{align*}
Using $c=\beta_{n}/\beta_{p}$, the above conditions can be rewritten as
\begin{align*}
(1-r)(1-2\delta) + r\left((1-\delta) \beta_{p} (1-2\nu) - \delta \beta_{n}\right) & > 0 \quad \text{and} \\
(1-r)(1-2\delta) + r\left((1-\delta) \beta_{n} - \delta \beta_{p} (1-2\nu)\right) & > 0.
\end{align*}

\noindent {\bf Step 2: If the given conditions on the bias parameters hold, then there cannot exist any $\lambda \in \R$ such that both the group-wise thresholds applied to $\eta(x, a)$ in $h_{\lambda}$ are at most $\delta$, or both the thresholds are at least $1-\delta$. (Thresholds lying on the same side)}
The threshold on $\eta(x, 1)$ is at most $\delta$ if and only if $2 - \dfrac{\lambda}{(1-r)q}  \geq \dfrac{1}{\delta}$, or equivalently, $\dfrac{-(1 - 2\delta)(1-r) q}{\delta} \geq \lambda$. Similarly, the threshold on $\eta(x, 0)$ is at most  $\delta$ if and only if $1 + \dfrac{1-2\nu}{c} + \dfrac{\lambda}{\beta_{n} r q} \geq \dfrac{1}{\delta}$, or equivalently, $\lambda \geq \beta_{n} rq \left(\dfrac{1-\delta}{\delta} - \dfrac{1-2\nu}{c}\right)$. If both were to hold simultaneously, then
$\dfrac{-(1 - 2\delta)(1-r) q}{\delta} \geq \beta_{n} rq \left(\dfrac{1-\delta}{\delta} - \dfrac{1-2\nu}{c}\right)$, or equivalently, $(1-r)(1-2\delta) + r\left((1-\delta) \beta_{n} - \delta \beta_{p} (1-2\nu)\right) \leq 0$, which would violate the second condition in Step 1. 

On the other hand, the threshold on $\eta(x, 1)$ is at least interval $1-\delta$ if and only if $\dfrac{1}{1-\delta} \geq 2 - \dfrac{\lambda}{(1-r)q}$, or equivalently, $\lambda \geq \dfrac{(1-2\delta)(1-r) q}{1-\delta}$. Similarly, the threshold on $\eta(x, 0)$ is at least $1- \delta$ if and only if $\dfrac{1}{1-\delta} \geq 1 + \dfrac{1-2\nu}{c} + \dfrac{\lambda}{\beta_{n} r q}$, or equivalently, $\beta_{n}rq \left(\dfrac{\delta}{1-\delta} - \dfrac{1-2\nu}{c}\right) \geq \lambda$. If both were to hold simultaneously, then $\beta_{n}rq \left(\dfrac{\delta}{1-\delta} - \dfrac{1-2\nu}{c}\right) \geq \dfrac{(1-2\delta)(1-r) q}{1-\delta}$, or equivalently, $(1-r)(1-2\delta) + r\left((1-\delta) \beta_{p} (1-2\nu) - \delta \beta_{n}\right) \leq 0$, which would violate the first condition in Step 1. \\

\noindent {\bf Step 3: For any $\lambda$, if the group-wise thresholds on $\eta(x, a)$ in $h_{\lambda}$ have one of them at most $\delta$ and another at least $1-\delta$, then $h_{\lambda}$ cannot satisfy equal opportunity unless $h_{\lambda}=h^{\text{rej}}$, whenever $h^{\text{rej}}(x, a) \neq \bot$. (Thresholds lying on opposite sides)} 
From our assumption, we know that $h^{\text{rej}}$ satisfies equal opportunity whenever $h^{\text{rej}}(x, a) \neq \bot$. Suppose the threshold of $h_{\lambda}$ on $\eta(x,0) (\text{call it } t_{0})$ and the threshold on $\eta(x,1) (\text{call it } t_{1})$ are separated by either $\delta$ or $1 - \delta$. WLOG assume $t_{0} < t_{1}$. Suppose $t_{0} \leq \delta \leq t_{1}$. Since there is no input $(x, a)$ with $\eta(x, a) \in (\delta, 1-\delta)$, we get $TPR_{0}(h_{\lambda}) \geq TPR_{0}(h^{*})$ and $TPR_{1}(h^{*}) \geq TPR_{1}(h_{\lambda})$. 
Since $h^{*}$ satisfies equal opportunity on the original distribution $D$, we have $TPR_{0}(h^{*}) = TPR_{1}(h^{*})$.
If $h_{\lambda}$ satisfies equal opportunity on the biased distribution $\tilde{D}$, we have $\widetilde{TPR}_{0}(h_{\lambda}) = \widetilde{TPR}_{1}(h_{\lambda})$, and since $\widetilde{TPR}_{a}(h_{\lambda}) = TPR_{a}(h_{\lambda})$ by Corollary 3.6, it implies $TPR_{0}(h_{\lambda}) = TPR_{1}(h_{\lambda})$. 
Combining the two observations above, we get $TPR_{0}(h_{\lambda}) = TPR_{0}(h^{*}) = TPR_{1}(h^{*}) = TPR_{1}(h_{\lambda})$, which is possible only if $h_{\lambda} \equiv h^{*}$ (as both are group-wise threshold classifiers). 
The other case $t_{0} \leq 1-\delta \leq t_{1}$ can be argued similarly. \\

\noindent {\bf Combing Steps 1, 2, and 3 to complete the proof:} Finally, from Steps 1, 2, and 3 together, it is clear that when the above conditions are met, then the only possible classifiers $\tilde{h}_{\lambda}$ are the ones for which $\tilde{h}_{\lambda}=h^{\text{rej}}$, whenever $h^{\text{rej}}(x, a) \neq \bot$, and they satisfy equal opportunity on $\tilde{D}$. Furthermore, by our assumption, $h^{\text{rej}}$ is EO-fair. 

\end{proof}

\subsection{Proofs for Section \ref{sec:robust}} \label{appndx:proof_eo-robust-recovery}

Definition \ref{def:eps-robust-property} of $\epsilon$-robustness of $h^{*}$ can be rewritten as
\begin{equation*}
    \begin{split}
         &h^{*}(x)=\underset{h \in \mathcal{H}}{\mathrm{argmax}}~\prob{h(X, A) = Y'} \\
& = \underset{h \in \mathcal{H}}{\mathrm{argmax}}~ \expecto{(X, A)}{h(X, A) (2\eta'(X, A)-1)} \\
& = \underset{h \in \mathcal{H}}{\mathrm{argmax}}~ \sum_{a=0}^{1} \prob{A=a} \expecto{X|A=a}{h(X, a) (2\eta'(X, a)-1)}\\
& = \underset{h \in \mathcal{H}}{\mathrm{argmax}}~ \sum_{a=0}^{1} \prob{A=a} \underset{X|A=a}{\mathrm{E}} \Bigg[h(X, a)
\left(2~ \frac{P_{a} \eta(X, a) + Q_{a}}{R_{a} \eta(X, a) + S_{a}} -1\right)\Bigg] \\
& = \underset{h \in \mathcal{H}}{\mathrm{argmax}}~ \sum_{a=0}^{1} \prob{A=a}\underset{X|A=a}{\mathrm{E}}\Bigg[h(X, a)
\frac{(2P_{a} - R_{a}) \eta(X, a) + (2Q_{a} - S_{a})}{R_{a} \eta(X, a) + S_{a}}\Bigg],
    \end{split}
\end{equation*}
for any $S_{a} = 1, |R_{a}| \leq \epsilon, |Q_{a}| \leq \epsilon, 1-\epsilon \leq P_{a} \leq 1+\epsilon$, and $P_{a}S_{a} - Q_{a}R_{a} \geq 0$.

\begin{proof} {\bf(Proof of Theorem \ref{thm:eo-robust-recovery})}
Let $\tilde{h}_{EO} = \underset{h \in \widetilde{\mathcal{H}}_{\text{fair}}}{\text{argmax}}~ \prob{h(X, A)=\tilde{Y}}$. By Lagrange duality, we can write
\begin{align*}
\tilde{h}_{EO} & = \underset{h \in \mathcal{H}}{\mathrm{argmax}}~ \underset{\lambda \in \R}{\min}~ \prob{h(X, A)=\tilde{Y}} + \lambda~ \prob{h(X, A)=1 | \tilde{Y}=1, A=0} - \lambda~ \prob{h(X, A)=1 | \tilde{Y}=1, A=1} \\
& = \underset{h \in \mathcal{H}}{\mathrm{argmax}}~ \underset{\lambda \in \R}{\min}~ \sum_{a=0}^{1} \prob{A=a}~ \expecto{X|A=a}{h(X, a) (2\tilde{\eta}(X, a) - 1)} + \frac{(-1)^{\id{a=1}}~ \lambda}{\expecto{X|A=a}{\tilde{\eta}{(X, a)}}}~ \expecto{X|A=a}{h(X, a)\tilde{\eta}(X, a)} \\
& = \underset{h \in \mathcal{H}}{\mathrm{argmax}}~ \underset{\lambda \in \R}{\min}~ \sum_{a=0}^{1} \prob{A=a}~ \expecto{X|A=a}{h(X, a) (2\tilde{\eta}(X, a) - 1)} + \frac{(-1)^{\id{a=1}}~ \lambda}{\prob{\tilde{Y}=1|A=a}}~ \expecto{X|A=a}{h(X, a)\tilde{\eta}(X, a)} \\
& = \underset{h \in \mathcal{H}}{\mathrm{argmax}}~ \underset{\lambda \in \R}{\min}~ \sum_{a=0}^{1} \prob{A=a}~ \expecto{X|A=a}{h(X, a) \left(\left(2 + \dfrac{(-1)^{\id{a=1}} \lambda}{\prob{\tilde{Y}=1, A=a}}\right) \tilde{\eta}(X, a) - 1\right)} \\
& = \underset{h \in \mathcal{H}}{\mathrm{argmax}}~ \underset{\lambda \in \R}{\min}~ r~ \expecto{X|A=0}{h(X, 0) \left(\left(2 + \frac{\lambda}{\beta_{p} (1-\nu)rq}\right) \frac{(1-\nu)\eta(X, 0)}{(1-c) \eta(X, 0) + c} - 1\right)} \\
& \qquad \qquad \qquad \qquad + (1-r)~ \expecto{X|A=1}{h(X, 1) \left(\left(2 - \frac{\lambda}{(1-r) q}\right) \eta(X, 1) - 1\right)} \\
& = \underset{h \in \mathcal{H}}{\mathrm{argmax}}~ \underset{\lambda \in \R}{\min}~ r~ \expecto{X|A=0}{h(X, 0)~ \frac{\left(2(1-\nu) + \dfrac{\lambda}{\beta_{p} rq} - (1-c)\right) \eta(X, 0) - c}{(1-c) \eta(X, 0) + c}} \\
& \qquad \qquad \qquad \qquad + (1-r)~ \expecto{X|A=1}{h(X, 1) \left(\left(2 - \dfrac{\lambda}{(1-r)q}\right) \eta(X, 1) - 1\right)} \\
& = \underset{h \in \mathcal{H}}{\mathrm{argmax}}~ \underset{\lambda \in \R}{\min}~ r~ \expecto{X|A=0}{h(X, 0)~ \frac{\left(1 - 2\nu + c + \dfrac{\lambda}{\beta_{p} rq} \right) \eta(X, 0) - c}{(1-c) \eta(X, 0) + c}} \\
& \qquad \qquad \qquad \qquad + (1-r)~ \expecto{X|A=1}{h(X, 1) \left(\left(2 - \dfrac{\lambda}{(1-r)q}\right) \eta(X, 1) - 1\right)} \\
& = \underset{h \in \mathcal{H}}{\mathrm{argmax}}~ \underset{\lambda \in \R}{\min}~ r~ \expecto{X|A=0}{h(X, 0)~ \frac{\left(1 + \dfrac{1- 2\nu}{c} + \dfrac{\lambda}{\beta_{n} rq}\right) \eta(X, 0) - 1}{\dfrac{1-c}{c} \eta(X, 0) + 1}} \\
& \qquad \qquad \qquad \qquad + (1-r)~ \expecto{X|A=1}{h(X, 1) \left(\left(2 - \dfrac{\lambda}{(1-r)q}\right) \eta(X, 1) - 1\right)}, \qquad \text{using $c = \dfrac{\beta_{n}}{\beta_{p}}$}.
\end{align*}
Now plugging in 
\begin{align*}
\frac{(2P_{0} - R_{0}) \eta(X, 0) + (2Q_{0} - S_{0})}{R_{0} \eta(X, 0) + S_{0}} & = \frac{\left(1 + \dfrac{1- 2\nu}{c} + \dfrac{\lambda}{\beta_{n} rq}\right) \eta(X, 0) - 1}{\dfrac{1-c}{c} \eta(X, 0) + 1} \qquad \text{and} \\
\frac{(2P_{1} - R_{1}) \eta(X, 1) + (2Q_{1} - S_{1})}{R_{1} \eta(X, 1) + S_{1}} & = \left(2 - \frac{\lambda}{(1-r) q}\right) \eta(X, 1) - 1,
\end{align*}
we derive
\begin{alignat*}{4}
P_{0} & = \dfrac{1-\nu}{c} + \dfrac{\lambda}{2 \beta_{n} rq}, \qquad && Q_{0} = 0, \qquad && R_{0} = \dfrac{1-c}{c}, \qquad && S_{0} = 1, \\
P_{1} & = 1 - \dfrac{\lambda}{2(1-r)q}, \qquad && Q_{1} = 0, \qquad && R_{1} = 0, \qquad && S_{1} = 1.
\end{alignat*}
\noindent {\bf Step 1: If the given conditions on the bias parameters hold, then there exists a $\lambda \in \R$ such that $1-\epsilon \leq P_{0}, P_{1} \leq 1+\epsilon$, and therefore, $h_{\lambda} = h^{*}$.}
They satisfy our $\epsilon$-robustness conditions if and only if
\begin{align*}
|R_{0}| & = \left|\dfrac{1-c}{c}\right| = \left|\frac{\beta_{p} - \beta_{n}}{\beta_{n}}\right| \leq \epsilon, \\
1-\epsilon \leq P_{0} & = \dfrac{1 - \nu}{c} + \dfrac{\lambda}{2\beta_{n} rq} \leq 1+\epsilon \\
1-\epsilon \leq P_{1} & = 1 - \dfrac{\lambda}{2(1-r)q} \leq 1+\epsilon.
\end{align*}
The above conditions also imply $P_{a}S_{a} - Q_{a}R_{a} \geq 0$, for $a \in \{0, 1\}$, so we do not need to include these  separately. The first inequality is equivalent to $(1-\epsilon)\beta_{n} \leq \beta_{p} \leq (1+\epsilon)\beta_{n}$. The second and third inequalities above are equivalent to $2rq \left((1-\epsilon)\beta_{n} - (1-\nu)\beta_{p}\right) \leq \lambda \leq 2rq \left((1+\epsilon)\beta_{n} - (1-\nu)\beta_{p}\right)$ and $-2\epsilon (1-r)q \leq \lambda \leq 2\epsilon (1-r)q$, respectively. There exists a $\lambda \in \R$ that satisfies the above two conditions on $\lambda$ simultaneously if and only if
\[
2rq \left((1-\epsilon)\beta_{n} - (1-\nu)\beta_{p}\right) \leq 2\epsilon (1-r)q \quad \text{and} \quad
-2\epsilon (1-r)q \leq 2rq \left((1+\epsilon)\beta_{n} - (1-\nu)\beta_{p}\right),
\]
or equivalently,
\[
r \left((1-\nu)\beta_{p} - (1-\epsilon)\beta_{n}\right) + \epsilon (1-r) \geq 0 \quad \text{and} \quad r \left((1+\epsilon)\beta_{n} - (1-\nu)\beta_{p}\right) + \epsilon (1-r) \geq 0.
\]
\noindent {\bf Step 2: If the given conditions on the bias parameters hold, then there cannot exist any $\lambda \in \R$ such that $\max\{P_{0}, P_{1}\} < 1-\epsilon$ or $\min\{P_{0} , P_{1}\} > 1+\epsilon$.} Suppose $\max\{P_{0}, P_{1}\} < 1-\epsilon$. Then 
\[
\dfrac{1 - \nu}{c} + \dfrac{\lambda}{2\beta_{n} rq} < 1-\epsilon \quad \text{and} \quad 1 - \dfrac{\lambda}{2(1-r)q} < 1-\epsilon.    
\]
Thus, $\lambda < 2rq \left((1-\epsilon)\beta_{n} - (1-\nu)\beta_{p}\right)$ and $\lambda > 2\epsilon(1-r)q$. This implies $2rq \left((1-\epsilon)\beta_{n} - (1-\nu)\beta_{p}\right) > 2\epsilon(1-r)q$, or equivalently, $r \left((1-\nu)\beta_{p} - (1-\epsilon)\beta_{n}\right) + \epsilon (1-r) < 0$, violating one of the conditions in the theorem. The case of $\min\{P_{0} , P_{1}\} > 1+\epsilon$ can be argued similarly.

\noindent {\bf Step 3: If the given conditions on the bias parameters hold, then an optimal $\lambda^{*} \in \R$ cannot correspond to $P_{0} < 1-\epsilon$ and $P_{1} > 1+\epsilon$ (or vice versa).} Suppose $\lambda^{*} \in \R$ satisfies $P_{0} < 1-\epsilon$ and $P_{1} > 1+\epsilon$, then
\[
P_{0} = \dfrac{1 - \nu}{c} + \dfrac{\lambda^{*}}{2\beta_{n} rq} < 1-\epsilon \quad \text{and} \quad P_{1} = 1 - \dfrac{\lambda^{*}}{2(1-r)q} > 1+\epsilon,
\]
which implies that $\lambda^{*} < 2rq \left((1-\epsilon)\beta_{n} - (1-\nu)\beta_{p}\right)$ and $\lambda^{*} < -2\epsilon(1-r)q$.
We know that
\[
\tilde{h}_{EO} = \underset{h \in \mathcal{H}}{\mathrm{argmax}}~ \prob{h(X, A)=\tilde{Y}} + \lambda^{*}~ \prob{h(X, A)=1 | \tilde{Y}=1, A=0} - \lambda^{*}~ \prob{h(X, A)=1 | \tilde{Y}=1, A=1}.
\]
Since we consider group-aware classification, when we know the optimal $\lambda^{*}$, the optimal equal opportunity classifier $\tilde{h}_{EO}$ on the biased distribution $\tilde{D}$ can be obtained by optimizing separately on the two groups. Note that $\tilde{h}_{EO}$ maximizes $\prob{h(X, A)=\tilde{Y}} + \lambda^{*}~ \prob{h(X, A)=1 | \tilde{Y}=1, A=0} - \lambda^{*}~ \prob{h(X, A)=1 | \tilde{Y}=1, A=1}$ over $h \in \mathcal{H}$, for the optimal $\lambda^{*} \in \R$. Since we consider group-aware classifiers, $\tilde{h}_{EO}$ maximizes the following objective on the underprivileged group $a=0$. 
\begin{align*}
& \prob{h(X, A)=\tilde{Y} | A=0} + \frac{\lambda^{*}}{\prob{A=0}}~ \prob{h(X, A)=1 | \tilde{Y}=1, A=0} \\
& = \prob{\tilde{Y}=1 | A=0}~ \widetilde{TPR}_{0}(h) + \prob{\tilde{Y}=0 | A=0}~ \widetilde{TNR}_{0}(h) + \frac{\lambda^{*}}{r}~ \widetilde{TPR}_{0}(h) \\ 
& = \prob{\tilde{Y}=1 | A=0}~ TPR_{0}(h) + \prob{\tilde{Y}=0, Y=0 | A=0}~ TNR_{0}(h) + \prob{\tilde{Y}=0, Y=1 | A=0}~ FNR_{0}(h) \\
& \qquad \qquad \qquad \qquad + \frac{\lambda^{*}}{r}~ TPR_{0}(h) \qquad \qquad \qquad \text{using Proposition \ref{prop:tilde-tpr-tnr}} \\
& = \frac{q\beta_{p}(1-\nu)}{q\beta_{p} + (1-q)\beta_{n}}~ TPR_{0}(h) + \frac{(1-q)\beta_{n}}{q\beta_{p} + (1-q)\beta_{n}}~ TNR_{0}(h) + \frac{q\beta_{p}\nu}{q\beta_{p} + (1-q)\beta_{n}}~ FNR_{0}(h) + \frac{\lambda^{*}}{r}~ TPR_{0}(h) \\
& = \left(\frac{q\beta_{p}(1-2\nu)}{q\beta_{p} + (1-q)\beta_{n}} + \frac{\lambda^{*}}{r}\right)~ TPR_{0}(h) + \frac{(1-q)\beta_{n}}{q\beta_{p} + (1-q)\beta_{n}}~ TNR_{0}(h) + \frac{q\beta_{p}\nu}{q\beta_{p} + (1-q)\beta_{n}} \\
& \hspace{9cm} \text{because $FNR_{0}(h) = 1 - TPR_{0}(h)$}. 
\end{align*}
Thus, $\tilde{h}_{EO}$ essentially maximizes a weighted linear combination of $TPR_{0}(h)$ and $TNR_{0}(h)$ over $h \in \mathcal{H}$, where the ratios of the weights for $TPR_{0}(h)$ and $TNR_{0}(h)$, respectively, is 
\begin{align*}
\frac{q\beta_{p}(1-2\nu)}{(1-q)\beta_{n}} + \frac{\lambda^{*} (q\beta_{p} + (1-q)\beta_{n})}{r(1-q)\beta_{n}} & \leq \frac{q (1-2\nu)}{(1-q) c} + \frac{\lambda^{*}}{r(1-q)\beta_{n}} \qquad \text{using $\beta_{p}, \beta_{n} \leq 1$} \\
& = \frac{q}{1-q} \left(\frac{1-2\nu}{c} + \frac{\lambda^{*}}{\beta_{n} rq}\right) \\
& = \frac{q}{1-q} \left(\frac{2(1-\nu)}{c} + \frac{\lambda^{*}}{\beta_{n}rq} - \frac{1}{c}\right) \\
& < \frac{q}{1-q} \left(2(1-\epsilon) - \frac{1}{c}\right) \qquad \text{using $P_{0} = \dfrac{1 - \nu}{c} + \dfrac{\lambda^{*}}{2\beta_{n} rq} < 1-\epsilon$} \\
& \leq \frac{q}{1-q}~ (1-\epsilon) \qquad \qquad \text{using $\left|\frac{1-c}{c}\right| \leq \epsilon$, and hence, $\frac{1}{c} \geq 1-\epsilon$}
\end{align*}
Hence, $TPR_{0}(\tilde{h}_{EO}) \leq TPR_{0}(h^{*})$. Similarly, $\tilde{h}_{EO}$ maximizes the following objective on the group $a=1$. 
\begin{align*}
& \prob{h(X, A)=\tilde{Y} | A=1} - \frac{\lambda^{*}}{\prob{A=1}}~ \prob{h(X, A)=1 | \tilde{Y}=1, A=1} \\
& = \prob{\tilde{Y}=1 | A=1}~ \widetilde{TPR}_{1}(h) + \prob{\tilde{Y}=0 | A=1}~ \widetilde{TNR}_{1}(h) - \frac{\lambda^{*}}{1-r}~ \widetilde{TPR}_{1}(h) \\ 
& = q \left(1 - \frac{\lambda^{*}}{(1-r)q}\right)~ TPR_{1}(h) + (1-q)~ TNR_{1}(h).
\end{align*}
Thus, $\tilde{h}_{EO}$ essentially maximizes a weighted linear combination of $TPR_{1}(h)$ and $TNR_{1}(h)$ over $h \in \mathcal{H}$, where the ratio of the weights for $TPR_{1}(h)$ and $TNR_{1}(h)$, respectively, is
\[
\frac{q}{1-q} \left(1 - \frac{\lambda^{*}}{(1-r)q}\right) > \frac{q}{1-q}~ (1+2\epsilon) \qquad \qquad \text{using $P_{1} = 1 - \frac{\lambda^{*}}{2(1-r)q} > 1+\epsilon$}.
\]
Hence, $TPR_{1}(\tilde{h}_{EO}) \geq TPR_{1}(h^{*})$. 

Combining the two observation $TPR_{0}(\tilde{h}_{EO}) \leq TPR_{0}(h^{*})$ and $TPR_{1}(\tilde{h}_{EO}) \geq TPR_{1}(h^{*})$ above and that $h^{*}$ satisfies equal opportunity, we get $TPR_{0}(\tilde{h}_{EO}) \leq TPR_{0}(h^{*}) = TPR_{1}(h^{*}) \leq TPR_{1}(\tilde{h}_{EO})$. However, $\tilde{h}_{EO} \in \widetilde{\mathcal{H}}_{\text{fair, EO}} = \mathcal{H}_{\text{fair, EO}}$ by Corollary \ref{corr:tpr_tnr_sp}, so we must have $TPR_{0}(\tilde{h}_{EO}) = TPR_{0}(h^{*}) = TPR_{1}(h^{*}) = TPR_{1}(\tilde{h}_{EO})$. This means that $\tilde{h}_{EO}$ would also maximize accuracy on the original distribution $D$, and therefore, by the $\epsilon$-robustness property, we must have $\tilde{h}_{EO} \equiv h^{*}$.

The case of $P_{0} > 1+\epsilon$ and $P_{1} < 1-\epsilon$ can be argued similarly.
\end{proof}

\subsection{Proofs for Section \ref{sec:time_varying}} \label{appndx:proofs-time-varying}

\begin{proof} {\bf(Proof for Theorem \ref{thm:uniform-time-varying-eo})} We start with the Bayes Optimal EO classifier at the time step $t$:

\begin{align*}
\tilde{h}_{EO}^{(t)}(x, a) & = \id{\tilde{\eta}_{t}(x, a) \geq \dfrac{1}{2 + \dfrac{(-1)^{\id{a=1}} \lambda^{*}}{\prob{\tilde{Y}_{t}=1, A=a}}}}
\end{align*}
From Proposition \ref{prop:time-eta-tilde} we can obtain the transformed $\tilde{\eta}_{t}$ with the same bias parameters at each time step:
\begin{align*}
& = \begin{cases} \id{\dfrac{\eta(x, 0)}{\dfrac{1-c}{1-\nu-c} \left(1 - \left(\dfrac{c}{1-\nu}\right)^{t}\right) \eta(x, 0) + \left(\dfrac{c}{1-\nu}\right)^{t}} \geq \dfrac{1}{2 + \dfrac{\lambda^{*}}{\prob{\tilde{Y_{t}}=1, A=0}}}}, \quad \text{for $a=0$} \\
\id{\eta(x, 1) \geq \dfrac{1}{2 - \dfrac{\lambda^{*}}{\prob{Y=1, A=1}}}}, \quad \text{for $a=1$} 
\end{cases} \\
& = \begin{cases} \id{\eta(x, 0) \geq \dfrac{1}{\dfrac{1-2\nu-c}{1-\nu-c} \left(\dfrac{1-\nu}{c}\right)^{t} + \dfrac{\lambda^{*}}{\beta_{n}^{t} rq} + \dfrac{1-c}{1-\nu-c}}}, \quad \text{for $a=0$, using $c = \beta_{n}/\beta_{p}$} \\
\id{\eta(x, 1) \geq \dfrac{1}{2 - \dfrac{\lambda^{*}}{(1-r)q}}}, \quad \text{for $a=1$}.
\end{cases}
\end{align*}
So the threshold on $\eta(x, 1)$ lies in the interval $(\delta, 1-\delta)$ if and only if $\delta < \left(2 - \dfrac{\lambda^{*}}{(1-r)q}\right)^{-1} < 1-\delta$, which is equivalent to $\dfrac{-(1 - 2\delta)(1-r) q}{\delta} < \lambda^{*} < \dfrac{(1-2\delta)(1-r) q}{1-\delta}$. Given this, now let's analyze the threshold on $\eta(x, 0)$. If $\beta_{n} < 1$, then as $t \rightarrow \infty$ we have $\lambda^{*}/(\beta_{n}^{t} rq) \rightarrow \infty$, and therefore, the threshold on $\eta(x, 0)$ in the above expression tends to $0$, i.e., it cannot remain within $(\delta, 1-\delta)$ interval. If $\beta_{n} = 1$ then $c = \beta_{n}/\beta_{p} > 1-\nu$. Thus, $\left((1-\nu)/c\right)^{t} \rightarrow 0$ as $t \rightarrow \infty$. Using this and $\beta_{n}=1$, the above expression becomes
\[
\dfrac{1}{\dfrac{1-2\nu-c}{1-\nu-c} \left(\dfrac{1-\nu}{c}\right)^{t} + \dfrac{\lambda^{*}}{\beta_{n}^{t} rq} + \dfrac{1-c}{1-\nu-c}} \rightarrow \dfrac{1}{\dfrac{\lambda^{*}}{rq} + \dfrac{1-c}{1-\nu-c}} \quad \text{as $t \rightarrow \infty$.}
\]
The threshold on $\eta(x, 1)$ lies in the interval $(\delta, 1-\delta)$ if and only if $\dfrac{-(1 - 2\delta)(1-r) q}{\delta} < \lambda^{*} < \dfrac{(1-2\delta)(1-r) q}{1-\delta}$. Similarly, the limit expression as $t \rightarrow \infty$ for threshold on $\eta(x, 0)$ lies in the interval $(\delta, 1-\delta)$ if and only if $\dfrac{1}{1-\delta} < \dfrac{\lambda^{*}}{rq} + \dfrac{1-c}{1-\nu-c} < \dfrac{1}{\delta}$, or equivalently, $rq \left(\dfrac{1}{1-\delta} - \dfrac{1-c}{1-\nu-c}\right) < \lambda^{*} < rq \left(\dfrac{1}{\delta} - \dfrac{1-c}{1-\nu-c}\right)$. There exists a $\lambda^{*}$ that simultaneously satisfies the constraints on both the thresholds if and only if
\[
\frac{-(1 - 2\delta)(1-r) q}{\delta} < rq \left(\frac{1}{\delta} - \frac{1-c}{1-\nu-c}\right) \quad \text{and} \quad
rq \left(\frac{1}{1-\delta} - \frac{1-c}{1-\nu-c}\right) < \frac{(1-2\delta)(1-r) q}{1-\delta}.
\]
Thus, the necessary conditions to recover $h^{*}$ using equal opportunity fair classification on the biased distribution after $t$ steps as $t \rightarrow \infty$ can be written as $\beta_{n}=1$ and
\[
\frac{r - (1-2\delta)(1-r)}{(1-\delta) r} < \frac{1-c}{1-\nu-c} < \frac{r + (1-2\delta)(1-r)}{\delta r}.
\]
The above conditions can be further simplified as $\beta_{n}=1$ and $1 - \dfrac{(1-2\delta)}{(1-\delta) r} < \dfrac{\nu \beta_{p}}{1-\beta_{p}(1 - \nu)} < 1 + \dfrac{(1 - 2\delta)}{\delta r}$.
\end{proof}

\begin{proof} (Proof for Theorem \ref{thm:eo-time-varying})
Similar to the proof of Theorem \ref{thm:blumStangl-reprove}, we begin with $\tilde{\eta}(x,a)$:

\begin{align*}
\tilde{h}_{EO, t}(x, a) & = \id{\tilde{\eta_{t}}(x, a) \geq \dfrac{1}{2 + \dfrac{(-1)^{\id{a=1}} \lambda^{*}}{\prob{\tilde{Y}=1, A=a}}}}
\end{align*}

Because the population in group $1$ is unaffected, we can obtain the conditions on $\lambda^{*}$. $\prob{Y=1, A=1} = (1-r)q$. From Lemma \ref{thm:eta-lemma}, we know that the threshold on $\eta(x, 1)$ lies in the interval $(\delta, 1-\delta)$ if and only if $\dfrac{1}{1-\delta} < 2 - \dfrac{\lambda^{*}}{(1-r)q}  < \dfrac{1}{\delta}$, or equivalently, $\dfrac{-(1 - 2\delta)(1-r)q}{\delta} < \lambda^{*} < \dfrac{(1-2\delta)(1-r)q}{1-\delta}$. We will now work towards obtaining the set of conditions on $\lambda^{*}$ by focusing on the quantities inside the indicator on $\tilde{\eta_{t}}(x,0)$. Using Proposition \ref{prop:time-eta-tilde}: 

\begin{align*}
\tilde{\eta}_{t}(x, 0) = \frac{\eta(x,0)}{{\displaystyle \sum_{i=1}^{t} \left(\dfrac{1 - c_{i}}{1 - \nu_{i}} 
\prod_{j=i+1}^{t} \dfrac{c_{j}}{1 - \nu_{j}}\right) \eta(x, 0) + \prod_{i=1}^{t} \dfrac{c_{i}}{1 - \nu_{i}}}} \geq \dfrac{1}{2 + \dfrac{\lambda^{*}}{\prob{\tilde{Y}=1, A=0}}}
\end{align*}

From previous derivations, we know that $\prob{\tilde{Y}=1, A=0} = rq\prod_{i=1}^{t}(\beta_{p, i}(1 - \nu_{i})$. Simplifying the above expression gives us the following:

\begin{align*}
    \eta(x,0) \geq \frac{1}{\displaystyle 2\prod_{i=1}^{t}\frac{1 - \nu_{i}}{c_{i}} + \frac{\lambda^{*}}{rq\prod_{i=1}^{t}\beta_{n,i}} - \sum_{i=1}^{t}\left( \frac{1 - c_{i}}{c_{i}} \prod_{j=1}^{i-1}\frac{1 - \nu_{j}}{c_{j}}\right)}
\end{align*}

where $\prod_{j=1}^{i-1}\frac{1 - \nu_{j}}{c_{j}} = 1 \text{ whenever } j > i$.

Using similar arguments as in Proposition \ref{thm:blumStangl-reprove}, the denominator must lie in the range $((1-\delta)^{-1}, \delta^{-1})$, which gives us the following inequalities on $\lambda^{*}$:

\begin{align} \label{lambda_star_1}
    \displaystyle
    \frac{\displaystyle rq\prod_{i=1}^{t}\beta_{n,i}}{1 - \delta} - 2rq\prod_{i=1}^{t} \left[ (1 - \nu_{i})\beta_{p,i} \right] + rq\left( \prod_{i=1}^{t} \left[ (1 - \nu_{i})\beta_{p,i} \right] \right) \sum_{j=1}^{t}\left( \frac{1 - c_{j}}{c_{j}} \prod_{k=1}^{j-1}\frac{1 - \nu_{k}}{c_{k}}\right) <  \lambda^{*}
\end{align}
\begin{center}
    and
\end{center}
\begin{align} \label{lambda_star_2}
    \displaystyle
    \lambda^{*} < \frac{\displaystyle rq\prod_{i=1}^{t}\beta_{n,i}}{\delta} - 2rq\prod_{i=1}^{t} \left[ (1 - \nu_{i})\beta_{p,i} \right] + rq\left( \prod_{i=1}^{t} \left[ (1 - \nu_{i})\beta_{p,i} \right] \right) \sum_{j=1}^{t}\left( \frac{1 - c_{j}}{c_{j}} \prod_{k=1}^{j-1}\frac{1 - \nu_{k}}{c_{k}}\right)
\end{align}

To get a tight bound, we can obtain an upper bound on the left side of Equation \ref{lambda_star_1} and a lower bound on the right side of Equation \ref{lambda_star_2}, using the assumed maximal bias of $\beta_{n,t} \in [\beta_n, 1], \beta_{p,t} \in [\beta_{p}, 1] \text{ and } \nu_{t} \in [0, \nu], \text{ where } \nu < \frac{1}{2}$. This gives us the following bounds on $\lambda^{*}$:

\begin{align}
    \displaystyle
    \frac{\displaystyle rq}{1 - \delta} - 2rq(1 - \nu)^{t}\beta_{p}^{t} + rq\frac{1 - \beta_{n}^{t}}{\beta_{n}^{t}} < \lambda^{*} < \frac{\displaystyle rq\beta_{n}^{t}}{\delta} - 2rq + rq (1 - \nu)^{t}\beta_{p}^{t} (\beta_{p} - 1)\frac{1 - \beta_{p}^t(1 - \nu)^{t}}{1 - \beta_{p}(1 - \nu)}
\end{align}

Comparing this with the conditions on $\lambda^{*}$ found earlier: $\dfrac{-(1 - 2\delta)(1-r)q}{\delta} < \lambda^{*} < \dfrac{(1-2\delta)(1-r)q}{1-\delta}$, we get the following inequalities:

\begin{align*}
    \frac{(1 - 2\delta)(1 - r)}{(1 - \delta)r} - \frac{\delta}{1 - \delta} > \frac{1}{\beta_{n}^{t}} - 2\beta_{p}^{t}(1 - \nu)^{t}
\end{align*}
\begin{center}
    and
\end{center}
\begin{align*}
    \frac{(1- 2\delta)(1 - r)}{\delta r} > (1 - \nu)^{t}\beta_{p}^{t}(1 - \beta_{p})\frac{1 - \beta_{p}^{t}(1 - \nu)^{t}}{1 - \beta_{p}(1 - \nu)} - \frac{\beta_{n}^{t}}{\delta} - 2
\end{align*}
\end{proof}

\begin{theorem} \label{thm:dp-time-varying}
Assuming boundedness on data bias parameters at each time step: $\beta_{n,t} \in [\beta_n, 1], \beta_{p,t} \in [\beta_{p}, 1] \text{ and } \nu_{t} \in [0, \nu], \text{ where } \nu < \frac{1}{2}$; whenever the following relationships hold:
\begin{align*}
(1 - 3\delta)(\beta_{p} - 1)\dfrac{1 - \beta_{p}^{t}(1 - \nu)^{t}}{1 - \beta_{p}(1 - \nu)} - 2\delta - (1 - \delta)\left( \dfrac{1 - \beta_{n}^{t}}{\beta_{n}^{t}} - 2\beta_{p}^{t}(1 - \nu)^{t}\right) & > 0 \quad \text{and} \\
1 + \delta\left( (\beta_{p} - 1)\dfrac{1 - \beta_{p}^{t}(1 - \nu)^{t}}{1 - \beta_{p}(1 - \nu)} - \dfrac{2}{\beta_{n}^{t}} - (2\delta - 1)\dfrac{1 - \beta_{n}^{t}}{\beta_{n}^{t}} \right) -2(\delta - 1) & > 0,
\end{align*}
we have $\tilde{h}_{DP, t}(x,a) = h_{DP}(x,a) = h^{*}$.
\end{theorem}

\begin{proof}
We again begin with the thresholding results on Demographic Parity from Proposition \ref{prop:biased-dp-threshold}.
\begin{align*}
\tilde{h}_{DP}(x, a) & = \id{\tilde{\eta}(x, a) \geq \frac{1}{2} - \frac{-1^{\id{a=1}} \lambda^{*}}{2}} \\
\text{(Using Proposition \ref{prop:time-eta-tilde})}
& = \begin{cases} \id{\dfrac{\eta(x,0)}{{\displaystyle \sum_{i=1}^{t} \left(\dfrac{1 - c_{i}}{1 - \nu_{i}} 
\prod_{j=i+1}^{t} \dfrac{c_{j}}{1 - \nu_{j}}\right) \eta(x, 0) + \prod_{i=1}^{t} \dfrac{c_{i}}{1 - \nu_{i}}}} \geq \dfrac{1}{2} - \dfrac{\lambda^{*}}{2}}, \quad \text{for $a=0$} \\
\id{\eta(x, 1) \geq \dfrac{1}{2} + \dfrac{\lambda^{*}}{2}}, \quad \text{for $a=1$}
\end{cases} \\
& = \begin{cases} \id{\eta(x, 0) \geq \dfrac{(1 - \lambda^{*})\prod_{i=1}^{t} \dfrac{c_{i}}{1 - \nu_{i}}}{2 - (1 - \lambda^{*})\sum_{i=1}^{t} \left(\dfrac{1 - c_{i}}{1 - \nu_{i}} 
\prod_{j=i+1}^{t} \dfrac{c_{j}}{1 - \nu_{j}}\right)}}, \quad \text{for $a=0$} \\
\id{\eta(x, 1) \geq \dfrac{1}{2} + \dfrac{\lambda^{*}}{2}}, \quad \text{for $a=1$}
\end{cases} \\
& = \begin{cases} \id{\eta(x, 0) \geq \dfrac{1}{\dfrac{2}{1 - \lambda^{*}} \prod_{i=1}^{t} \dfrac{1 - \nu_{i}}{c_{i}} - \sum_{i=1}^{t}\left( \dfrac{1 - c_{i}}{c_{i}} \prod_{j=1}^{i-1}\dfrac{1 - \nu_{j}}{c_{j}}\right)}}, \quad \text{for $a=0$} \\
\id{\eta(x, 1) \geq \dfrac{1}{2} + \dfrac{\lambda^{*}}{2}}, \quad \text{for $a=1$}
\end{cases} \\
\end{align*}

where $\prod_{j=1}^{i-1}\frac{1 - \nu_{j}}{c_{j}} = 1 \text{ whenever } j > i$. From Lemma \ref{thm:eta-lemma}, we know that the threshold on $\eta(x, 1)$ lies in the interval $(\delta, 1-\delta)$ if and only if $\delta < \dfrac{1}{2} + \dfrac{\lambda^{*}}{2}  < 1 - \delta$, or equivalently, $2\delta - 1 < \lambda^{*} < 1 - 2\delta$. Similarly, the threshold on $\eta(x, 0)$ lies in the interval $(\delta, 1- \delta)$ if and only if the denominator lies in the range $((1-\delta)^{-1}, \delta^{-1})$ which gives us the following inequalities on $\lambda^{*}$:
\begin{align} \label{left_lambda*}
\dfrac{1 + (1 - \delta)\left( \sum_{i=1}^{t}\left( \dfrac{1 - c_{i}}{c_{i}} \prod_{j=1}^{i-1}\dfrac{1 - \nu_{j}}{c_{j}}\right) - 2\prod_{i=1}^{t} \dfrac{1 - \nu_{i}}{c_{i}} \right)}{1 + (1 - \delta)\sum_{i=1}^{t}\left( \dfrac{1 - c_{i}}{c_{i}} \prod_{j=1}^{i-1}\dfrac{1 - \nu_{j}}{c_{j}}\right)} < \lambda^{*}
\end{align}
\begin{align} \label{right_lambda*}
\lambda^{*} < \dfrac{1 + \delta \left( \sum_{i=1}^{t}\left( \dfrac{1 - c_{i}}{c_{i}} \prod_{j=1}^{i-1}\dfrac{1 - \nu_{j}}{c_{j}}\right) - 2\prod_{i=1}^{t} \dfrac{1 - \nu_{i}}{c_{i}} \right)}{1 + \delta\sum_{i=1}^{t}\left( \dfrac{1 - c_{i}}{c_{i}} \prod_{j=1}^{i-1}\dfrac{1 - \nu_{j}}{c_{j}}\right)}
\end{align}

To get a tight bound, we can obtain an upper bound on Equation \ref{left_lambda*} and a lower bound on Equation \ref{right_lambda*}, using the assumed maximal bias of $\beta_{n,t} \in [\beta_n, 1], \beta_{p,t} \in [\beta_{p}, 1] \text{ and } \nu_{t} \in [0, \nu], \text{ where } \nu < \frac{1}{2}$. This gives us the following bounds on $\lambda^{*}$:

\begin{align*}
\dfrac{1 + (1 - \delta)\left( \dfrac{1 - \beta_{n}^{t}}{\beta_{n}^{t}} - 2\beta_{p}^{t}(1 - \nu)^{t} \right)}{1 + (1 - \delta)(\beta_{p} - 1)\left( \dfrac{1 - \beta_{p}^{t}(1 - \nu)^{t}}{1 - \beta_{p}(1 - \nu)} \right)} < 1 - 2\delta \text{ , and } \\
2\delta - 1 < \dfrac{1 + \delta\left( (\beta_{p} - 1)\left( \dfrac{1 - \beta_{p}^{t}(1 - \nu)^{t}}{1 - \beta_{p}(1 - \nu)} \right) - \dfrac{2}{\beta_{n}^{t}} \right)}{1 + \delta \dfrac{1 - \beta_{n}^{t}}{\beta_{n}^{t}}}
\end{align*}

Therefore, to have a $\lambda^{*}$ which satisfies both the sets of the inequalities, the following conditions must hold:
\begin{center}
    \begin{align}
        (1 - 3\delta)(\beta_{p} - 1)\dfrac{1 - \beta_{p}^{t}(1 - \nu)^{t}}{(1 - \beta_{p}(1 - \nu)} - 2\delta - (1 - \delta)\left( \dfrac{1 - \beta_{n}^{t}}{\beta_{n}^{t}} - 2\beta_{p}^{t}(1 - \nu)^{t}\right) > 0
    \end{align}
    and
    \begin{align}
        1 + \delta\left( (\beta_{p} - 1)\dfrac{1 - \beta_{p}^{t}(1 - \nu)^{t}}{(1 - \beta_{p}(1 - \nu)} - \dfrac{2}{\beta_{n}^{t}} - (2\delta - 1)\dfrac{1 - \beta_{n}^{t}}{\beta_{n}^{t}} \right) -2(\delta - 1) > 0
    \end{align}
\end{center}
    
\end{proof}

\end{document}